\documentclass{aims} 
\usepackage{amsmath}
\usepackage{paralist}
\usepackage[authoryear]{natbib}
\usepackage{mathrsfs}
\usepackage{mathtools}
\usepackage[misc]{ifsym}
\usepackage{epsfig} 
\usepackage{epstopdf} 
\usepackage[colorlinks=true]{hyperref}
\hypersetup{urlcolor=blue, citecolor=red}
\usepackage[linesnumbered,ruled,vlined]{algorithm2e}
\usepackage{url}
\usepackage{algorithmic}
\usepackage{caption}
\usepackage{subfigure,graphicx}
\usepackage{comment}
\usepackage{bm}
\usepackage{tabularx}
\allowdisplaybreaks

\def\currentvolume{X}
 \def\currentissue{X}
  \def\currentyear{200X}
   \def\currentmonth{XX}
    \def\ppages{X-XX}
     \def\DOI{10.3934/xx.xxxxxxx}

\newtheorem{theorem}{Theorem}[section]
\newtheorem{corollary}[theorem]{Corollary}

\theoremstyle{definition}

\DeclareMathOperator{\dive}{div}
\DeclareMathOperator{\DKL}{D_{KL}}

\newcommand{\rev}[1]{\textcolor{black}{#1}}

\title[Diffusion map particle systems for generative modeling]
{Diffusion map particle systems for generative modeling} 

\author[Fengyi Li and Youssef Marzouk]{}


\thanks{$^*$Corresponding author: Fengyi Li}

\begin{document}
\maketitle

\centerline{\scshape
Fengyi Li$^{{\href{fengyil@mit.edu}{\textrm{\Letter}}}*}$
and Youssef Marzouk$^{{\href{ymarz@mit.edu}{\textrm{\Letter}}}}$}

\medskip

{\footnotesize
 \centerline{Massachusetts Institute of Technology, Cambridge, MA 02139 USA}
} 

\bigskip

 \centerline{(Communicated by Handling Editor)}


\begin{abstract}
We propose a novel diffusion map particle system (DMPS) for generative modeling, based on diffusion maps and Laplacian-adjusted Wasserstein gradient descent (LAWGD). Diffusion maps are used to approximate the generator of the corresponding Langevin diffusion process from samples, and hence to learn the underlying data-generating manifold. On the other hand, LAWGD enables efficient sampling from the target distribution given a suitable choice of kernel, which we construct here via a spectral approximation of the generator, computed with diffusion maps. Our method requires no offline training and minimal tuning, and can outperform other approaches on data sets of moderate dimension.
\end{abstract}


\section{Introduction}\label{sec:intro}

Generative modeling is a central task in fields such as computer vision \citep{cv1,cv2,cv3,cv4,cv5,cv6} and natural language processing \citep{nlp1,nlp2,nlp3,nlp4}, and applications ranging from medical image analysis \citep{med_img1, med_img2, med_img3, med_img4} to protein design \citep{protein1, protein2}. 
\rev{Given samples from a probability distribution of interest, the goal of generative modeling is to generate additional samples from the same distribution, without knowledge of its unnormalized density.}
Despite their successes, popular generative models such as variational auto-encoders (VAE) \citep{vae1, vae2}, generative adversarial networks (GAN) \citep{gan}, and diffusion models or score-based generative models (SGM) \citep{diffusion1, diffusion2, diffusion3, diffusion4}, typically need careful hyperparameter tuning \citep{hyper1, hyper2} and may involve long convergence times, e.g., for Langevin-type sampling \citep{diffusion_time}. The performance of such methods highly depends on the architecture and the choice of parameters of deep neural networks \citep{parameter_tuning, FinetuningGM}, which, all too often, require expert knowledge.

In this paper, we propose a new nonparametric kernel-based approach to generative modeling, based on diffusion maps and interacting particle systems. 


Diffusion maps \citep{diffusion_map1, diffusion_map2, diffusion_map3, diffusion_map_approx}, along with many other graph-based methods \citep{eigenmaps, isomap, LLE}, have mainly been used as a tool for nonlinear dimension reduction. The kernel matrix, constructed using pairwise distances between samples with proper normalization, approximates the generator of a Langevin diffusion process. This approximation becomes exact as the number of samples goes to infinity and the kernel bandwidth goes to zero. Construction of the kernel matrix using smooth kernels (e.g., Gaussians) enables one to compute the inverse of the eigenvalues and the gradients of the eigenfunctions analytically. 

Separately, the notion of gradient flow \citep{GF1,GF2, GF3} underlies a very active field of research and offers a unifying perspective on sampling and optimization \citep{jko, Wibisono2018Sampling}, with numerous connections to partial differential equations and differential geometry \citep{Villani2008OptimalTO, GF_PDE1}. Many common sampling algorithms approximate the gradient flow (of some functional) on a space of probability measures. The unadjusted Langevin algorithm (ULA) is a canonical example, and follows from the time discretization of a Langevin SDE. But many other particle systems, particularly interacting particle systems, approximate gradient flows: examples include Stein variational gradient descent (SVGD) \citep{svgd1, Liu2016SVGD}, affine-invariant interacting Langevin dynamics (ALDI) \citep{ALDI}, and Laplacian-adjusted Wasserstein gradient descent (LAWGD) \citep{Chewi2020SVGD}. ULA, SVGD, and other algorithms access the target distribution via the gradient of its log-density, and hence many of these methods have become popular for Bayesian inference.


Our approach combines these two ideas by using diffusion maps to directly approximate (the gradient of the inverse of) the generator of the Langevin diffusion process from samples, and in turn using this approximation within LAWGD to produce more samples efficiently. \rev{The only input is a set of samples from the target distribution, and thus our approach enables generative modeling.}
Compared to other generative modeling methods, our approach has several advantages. First, the use of diffusion maps facilitates accurate sampling from distributions supported on manifolds, particularly when the dimension of the manifold is lower than that of the ambient space. Second, we demonstrate accurate sampling from distributions with (a priori unknown) bounded support. Both of these features are in contrast with methods driven only by approximations of the local score: we conjecture that such methods are less able to detect the overall geometry of the target distribution, whereas our approach harnesses graph-based methods that are widely used for nonlinear dimension reduction to approximate the generator as a whole. Finally, our method is quite simple and computationally efficient comparing to training a neural network (e.g., as in score-based generative modeling) \citep{diffusion2} or normalizing flows \citep{reg_flow_Caterini, flow_ho}: the only parameter that needs to be tuned is the kernel bandwidth, and no offline training is required. 



\section{Notation and preliminaries}\label{sec:notation}

We use $\pi$ to denote the target distribution, i.e., the distribution we would like to sample from, \rev{and assume it is supported on a compact manifold $\mathcal D$ embedded in $\mathbb R^d$.} We let $V: \mathcal{D} \to \mathbb{R}$ denote the potential associated with this distribution, such that $\pi \propto \exp(-V)$. \rev{All distributions are assumed to have densities with respect to Lebesgue measure on $\mathbb{R}^d$, and we abuse notation by using the same symbol to denote a measure and its density.} The kernel $K(\cdot, \cdot): \mathcal D \times \mathcal D\rightarrow \mathbb R$ is assumed to be differentiable with respect to both arguments, and we use $\nabla_1 K(\cdot, \cdot)$, $\nabla_2 K(\cdot, \cdot)$ to denote the (Euclidean) gradient of the kernel with respect to its first and second arguments, respectively. We assume sufficient regularity to exchange the order of integration and differentiation (Leibniz integral rule) throughout. 
 

\section{A review of Wasserstein gradient flow and LAWGD} \label{sec:GF and LAWGD}
We first review the basics of Wasserstein gradient flow and recall the LAWGD algorithm of \citet{Chewi2020SVGD}.
\subsection{Gradient flow on Wasserstein space}
Let $F(\mu)$ be a functional over the space of probability measures, i.e., $F: \mathcal P_2(\mathbb R^d) \rightarrow \mathbb R$, \rev{where $\mathcal P_2$ is the space of probability measures that have finite second moment.} We seek to steer the measure $\mu_t$ (at time $t$) in the direction of steepest descent, defined by $F$ and a chosen metric. That is, $ \frac{\partial \mu_t}{\partial t} = -\nabla_{W_2} F(\mu_t)$, where $ \nabla_{W_2}$ denotes the general gradient in Wasserstein metric. Under some smoothness assumptions, we can write this as
\begin{align}\label{GF}
    \frac{\partial \mu_t}{\partial t} = \dive(\mu_t \nabla \delta F(\mu_t)),
\end{align}
where $\delta F(\mu)$ is the first variation of $F$ evaluated at $\mu$ \citep{OT_villani}. If we choose the functional $F$ to be the Kullback--Leibler (KL) divergence, $F(\mu) = \DKL(\mu||\pi)$, then \eqref{GF} becomes 
\begin{align*}
    \frac{\partial \mu_t}{\partial t} = \dive \left ( \mu_t \nabla \log \frac{d\mu_t}{d\pi} \right ),
\end{align*}
which is the Fokker--Planck equation \citep{jko}. The measure $\mu_t$ can be approximated by particles evolving according to the following dynamic,
\begin{align*}
    \dot{x} = - \nabla \log \frac{d\mu_t}{d\pi}(x).
\end{align*}
Forward Euler discretization with stepsize $h$ then yields the following numerical scheme,
\begin{align*}
    x_{t+1} = x_t - h \nabla \log \frac{d\mu_t}{d\pi}(x_t).
\end{align*}

\subsection{LAWGD algorithm}
One challenge with the scheme above is that the measure $\mu_t$ is intractable at time $t$. To solve this problem, SVGD \citep{Liu2016SVGD} implements the following kernelized dynamics (in the continuum limit),
\begin{align*}
    \dot{x} &= -\int  K(x, y)\nabla \log \frac{d \mu_t}{d\pi}(y)d\mu_t(y).
\end{align*} 
The expression above can be equivalently written as 
\begin{align}\label{svgd_sample}
    \dot{x}= -\int K(x, y) \nabla\frac{d\mu_t}{d\pi}(y) d\pi(y).
\end{align}
Define 
\begin{align*}
    \mathcal K_{\pi}  f (x) \coloneqq \int  K(x, y)f(y) d\pi(y).
\end{align*}
Then we write \eqref{svgd_sample} as
\begin{align*}
    \dot{x} = -\mathcal K_{\pi} \nabla \frac{d\mu_t}{d\pi}(x),
\end{align*}
and under SVGD, the density evolves according to 
\begin{align*}
    \partial_t \mu_t = \dive \left (\mu_t \mathcal K_{\pi} \nabla\frac{d\mu_t}{d\pi} \right ).
\end{align*}
On the other hand, LAWGD makes the JKO scheme implementable by considering the following kernelization
\begin{align*}
    \dot{x} = - \nabla \mathcal K_{\pi} \frac{d\mu_t}{d\pi}(x),
\end{align*}
and expanding it to obtain
\begin{align}
    \dot{x} &= -\int \nabla_1 K(x,y) \frac{d\mu_t}{d\pi}(y) d\pi(y)
    = -\int \nabla_1 K(x,y) d\mu_t(y).\label{LAWGD_dynamics}
\end{align}
\rev{The kernel $K(\cdot, \cdot)$ is specifically chosen in such a way that  $\mathcal K_\pi = \mathscr{L}^{-1}$}, where $-\mathscr L= \nabla^2  - \langle \nabla V, \nabla \cdot \rangle$ is generator of the Langevin diffusion process, \rev{and $\nabla^2$ denotes the Laplacian operator.} Here we assume $\mathscr L$ has discrete spectrum (see Appendix \ref{appendix:spectrum}). This choice is motivated by the rate of change of KL divergence, 
\begin{align}\label{conv_rate}
     \partial_t \DKL(\mu_t || \pi) = -\mathbb E_\pi \left[ \frac{d\mu_t}{d\pi} \mathscr L \mathcal K_\pi \frac{d\mu_t}{d\pi}\right].
\end{align}
Indeed, such a choice yields
\begin{align}
    \partial_t \DKL(\mu_t||\pi) &
    = -\mathbb E_\pi\left[ \left(\frac{d\mu_t}{d\pi} -1 \right)^2\right] = -\chi^2(\mu_t|| \pi).\label{chi_rate}
\end{align}
The evolution of the density under LAWGD thus follows
\begin{align*}
    \partial_t \mu_t = \dive \left (\mu_t  \nabla\mathscr L^{-1}\frac{d\mu_t}{d\pi} \right ).
\end{align*}
\rev{Here $\mathscr L^{-1}$ is the pseudo-inverse of the generator of the Langevin diffusion process. In practice, we write $\mathscr{L}$ using its spectral decomposition and define $\mathscr{L}^{-1}$ using all of the eigenpairs of $\mathscr{L}$ except for the one corresponding to the zero eigenvalue. See Appendix~\ref{appendix:spectrum} for more details.} Now suppose we initialize $\{x_0^i\}_{i = 1}^M \sim \mu_0$. We then obtain a discrete algorithm from \eqref{LAWGD_dynamics}, where the update step reads
\begin{align}\label{update_step}
    x_{t+1}^{i} = x_{t}^i - \frac h M \sum_{j = 1}^M \nabla_1 K_{\mathscr L^{-1}} (x_{t}^i, x_{t}^j). 
\end{align}
Here $K_{\mathscr L^{-1}}$ can be understood as a kernelized version of $\mathscr L^{-1}$, satisfying $\mathscr L^{-1} f(x)=  \int K_{\mathscr L^{-1}}(x,y) f(y) d\pi(y)$. In particular, setting $f = \frac{d\mu_t}{d\pi}$, we have $\mathscr L^{-1} \frac{d\mu_t}{d\pi}(x)=  \int K_{\mathscr L^{-1}}(x,y) \frac{d\mu_t}{d\pi}(y) d\pi(y) =  \int K_{\mathscr L^{-1}}(x,y) d\mu_t(y).$ More details about LAWGD can be found in \citet{Chewi2020SVGD}. 

\section{Diffusion map and kernel construction} \label{sec:DM}
\subsection{Kernel approximation of \texorpdfstring{$\mathscr L$}{the generator}} \label{subsec:kernel_L}
It remains to see how to implement
\eqref{update_step} in the generative modeling setting. 
Given a finite collection of samples $\{z^i\}_{i = 1}^n \sim \pi$, our goal is to approximate  $\nabla\mathscr L^{-1} f(x)$. Diffusion maps \citep{diffusion_map1, diffusion_map2, diffusion_map3, diffusion_map_approx} provide a natural framework for approximating $\mathscr L$ using kernels. Consider the Gaussian kernel $K_\epsilon(x,y) = e^{-\frac{\left\Vert x - y \right\Vert^2}{2\epsilon}}$ under some normalization
\begin{align*}
    M_\epsilon(x,y):=\frac{K_\epsilon(x,y)}{\sqrt{\int K_\epsilon(x,y) d\pi(x)} \sqrt{\int K_\epsilon(x,y) d\pi(y)}}.
\end{align*}
We construct the following two kernels,
\begin{align*}
    P_{\epsilon}^f(x,y) := \frac{M_\epsilon(x,y)}{\int M_\epsilon(x,y) d\pi(x)},\\
    P_{\epsilon}^b(x,y) := \frac{M_\epsilon(x,y)}{\int M_\epsilon(x,y) d\pi(y)},
\end{align*}
by normalizing with respect to the first or second argument. Here $f$ and $b$ stand for `forward' and `backward', respectively. Their actions on a function $g$ are defined as 
\begin{align*}
    T^{f,b}_\epsilon g(\cdot) = \int P^{f,b}_\epsilon(\cdot,y) g(y) d\pi(y).
\end{align*}
We also define the associated operators
\begin{align*}
    L^{f,b}_\epsilon :=\frac{\mathrm{Id} - T^{f,b}_\epsilon }{\epsilon}.
\end{align*}
As studied in \citet{diffusion_map_approx}, both the forward and backward operators converge to the generator \rev{in an appropriate sense}, 
\begin{align*}
    \lim_{\epsilon \rightarrow 0} L^{f}_\epsilon = \lim_{\epsilon \rightarrow 0}L^{b}_\epsilon = \mathscr L.
\end{align*}
Combining the previous results, we have
\begin{align}
   \lim_{\epsilon \rightarrow 0}\frac{g(\cdot) - \int P^{f,b}_\epsilon(\cdot,y) g(y) d\pi(y) }{\epsilon} = \mathscr L g(\cdot). \label{approx_kernel}
\end{align}
Note that this approximation holds only when data lie on a compact manifold \citep{diffusion_map_approx,Hein_rate, Singer_rate}. In practice, however, this assumption can be relaxed. In addition, although $L^{f,b}_{\epsilon}$ converges to a symmetric operator in the continuum limit, neither $L^f_{\epsilon}$ nor $L^b_{\epsilon}$ is symmetric. However, they satisfy $P^f_{\epsilon} (x,y)= P^b_{\epsilon}(y,x)$ and $L^f_{\epsilon} = (L^b_{\epsilon})^\ast$. Therefore, one way to get a symmetric operator is to take the average of the two
\begin{align*}
    L_{\epsilon} = \frac 12 (L_{\epsilon}^f + L_{\epsilon}^b).
\end{align*}
$L_{\epsilon}$ inherits all the properties of the forward and the backward kernel, hence converging to $\mathscr L$ in the limit. We now consider a finite sample approximation of the operator $\mathscr L$. Given samples $\{z^i\}_{i = 1}^N\sim \pi$, the above construction can be approximated by samples, namely, by replacing the integral by its empirical average (see Appendix \ref{appendix:finite_sample_kernel}). We add another subscript $N$ to these kernels, i.e., $M_{\epsilon, N}, P^{f,b}_{\epsilon, N}, T^{f,b}_{\epsilon, N}$ and $L^{f,b}_{\epsilon, N}$, to denote their counterparts resulting from finite sample approximation, \rev{where
\begin{align*}
    M_{\epsilon,N}(x, y) &= \frac{K_\epsilon(x,y)}{\sqrt{\sum_{i = 1}^N K_\epsilon(z^i,y)} \sqrt{\sum_{i=1}^N K_\epsilon(x,z^i)}},\\
    P_{\epsilon, N}^f(x,y) &= \frac{M_{\epsilon,N}(x,y)}{\sum_{i=1}^N M_{\epsilon,N}(z^i,y) },\\
    P_{\epsilon, N}^b(x,y) &= \frac{M_{\epsilon,N}(x,y)}{\sum_{i=1}^N M_{\epsilon,N}(x,z^i)},
\end{align*}
and 
\begin{align}
    P_{\epsilon, N}(x,y) = \frac 12 \left(P_{\epsilon, N}^f(x,y) + P_{\epsilon, N}^b(x,y)\right). \label{finite_P_N}
\end{align}
}

\subsection{Spectral approximation of \texorpdfstring{$\nabla \mathscr L^{-1} f(x)$}{the gradient of the inverse of the generator}}

We now exploit the spectral properties of the kernel. Recall that the operator $\mathscr L$ admits the following kernel approximation
\begin{align*}
    \mathscr L = \lim_{\epsilon \rightarrow 0} \frac{  \mathrm{Id} - T_\epsilon}{\epsilon},
\end{align*}
where $T_\epsilon g(\cdot) = \int P_\epsilon(\cdot, y) g(y) \pi(y)$.
\rev{Note that  $P_\epsilon = \frac 12 (P_\epsilon^f + P_\epsilon^b)$ is a symmetric positive definite kernel (see Appendix~\ref{pd_kernel} for a proof).} The classic tool for studying such a kernel is Mercer's theorem.
\begin{theorem}[Mercer]\label{mercer}
    Let $T_\epsilon$ and $P_\epsilon(x,y)$ be defined as in the previous discussion. Then there is a sequence of non-negative eigenvalues $\{\lambda_i\}_{i \in \mathbb N}$ and an orthonormal basis of eigenfunctions $\{\phi_i\}_{i \in \mathbb N}$ of $T_\epsilon$, i.e.,
    \begin{align*}
        & \int P_\epsilon(x,y) \phi_m(y) d\pi(y) = \lambda_m \phi_m(x),\\
        & \int \phi_m(x) \phi_n(x) d\pi(y) = \delta_{m,n},
    \end{align*}
    such that 
    \begin{align}\label{eigen_P}
        P_\epsilon(x,y) = \sum_{i = 1}^\infty \lambda_i \phi_i(x) \phi_i(y).
    \end{align}
\end{theorem}
An immediate corollary we can see is the connection between the spectra of $T_\epsilon$ and $L_\epsilon$.
\begin{corollary}
    Let $T_\epsilon$ and $L_\epsilon$ be defined as in \rev{the} previous discussion, and denote the eigenvalues and eigenfunctions of $T_\epsilon$ by $\{\lambda_i, \phi_i\}$, with $1 = \lambda_0 \geq \lambda_1  \geq \lambda_2 \geq  \cdots$. Then the set of eigenvalues and eigenfunctions of $L_\epsilon$ is $\left \{ \frac{1-\lambda_i }{\epsilon}, \phi_i \right \}$. In particular $\mathscr L$ can be written as the limit of the integral operator,
    \begin{align*}
        \mathscr L f(x)= \lim_{\epsilon \rightarrow 0}\int K_{\mathscr L, \epsilon}(x,y) f(y) d\pi(y),
    \end{align*}
    where $K_{\mathscr L, \epsilon}(x,y) = \sum_{i = 1}^\infty \frac{1-\lambda_i}{\epsilon} \phi_i(x) \phi_i(y)$.
\end{corollary}

The interchange of the order between limit and the integral is guaranteed by the dominated convergence theorem. Using Mercer's theorem and its corollary, we can write $K_{\mathscr L, \epsilon}$ and its inverse using the eigendecomposition:
\begin{align*}
    K_{\mathscr L, \epsilon} (x,y) &= \sum_{i = 1}^\infty \frac{1 - \lambda_i}{\epsilon} \phi_i(x)\phi_i(y),\\
    K_{\mathscr{L}^{-1}, \epsilon} (x,y) &=
    \sum_{i = 1}^\infty \frac{\epsilon}{1 - \lambda_i} \phi_i(x)\phi_i(y).
\end{align*}

$\mathscr L^{-1}$ admits the following kernel expression,
\begin{align*}
    \mathscr L^{-1} f(x) = \lim_{\epsilon \rightarrow 0} \int K_{\mathscr{L}^{-1}, \epsilon} (x,y) f(y) d\pi(y).
\end{align*}
For the case where $\nabla_1 K_{\mathscr{L}, \epsilon}(x,y)$ exists, we have
\begin{align*}
    \nabla_1 K_{\mathscr{L}^{-1}, \epsilon} (x,y) = \sum_{i = 1}^\infty \frac{\epsilon}{1 - \lambda_i} \nabla \phi_i(x)\phi_i(y). 
\end{align*}
Therefore, 
\begin{align*}
    \nabla L^{-1}_\epsilon f(x) = \int  \nabla_1 K_{\mathscr{L}^{-1}, \epsilon} (x,y) f(y) d\pi(y).
\end{align*}
Under regularity assumptions (see Appendix \ref{regularity_assump}), $\lim_{\epsilon \rightarrow 0} \nabla L_\epsilon f(x) = \nabla \mathscr L f(x)$ and $\lim_{\epsilon \rightarrow 0} \nabla L^{-1}_\epsilon f(x)= \nabla \mathscr L^{-1} f(x)$.

\section{The generative model}\label{sec:generative model}
\subsection{Computing \texorpdfstring{$\nabla_1  K_{\mathscr L^{-1},\epsilon,N}(x,y)$}{the gradient of the kernel} for arbitrary points \texorpdfstring{$x,y$}{}}
We can now write the update step \eqref{update_step} using an $\epsilon$ kernel approximation
\begin{align*}
    x_{t+1}^{i} = x_{t}^i - \frac hM \sum_{j = 1}^M \nabla_1 K_{\mathscr{L}^{-1}, \epsilon} (x_{t}^i, x_{t}^j),
\end{align*}
or using its empirical counterparts,
\begin{align*}
    x_{t+1}^{i} = x_{t}^i - \frac hM \sum_{j = 1}^M \nabla_1 K_{\mathscr{L}^{-1}, \epsilon, N} (x_{t}^i, x_{t}^j).
\end{align*}
Note that the kernel $K_{\mathscr{L}^{-1}, \epsilon,N}$ is constructed only at the locations of the training samples $\{z^i\}_{i = 1}^N \sim \pi$. In order to obtain an implementable algorithm, we need to be able to compute $\nabla_1 K_{\mathscr{L}^{-1}, \epsilon,N}(\cdot,\cdot)$ for arbitrary points $x^*,y^*$. One way is to interpolate the eigenfunctions $\phi$ and their gradients $\nabla\phi$ at $x^*$ and $y^*$. However, this is restricted by the number of training samples for learning the kernel, as well as the interpolation method. To overcome this problem, we propose yet another natural way of computing $\nabla_1 K_{\mathscr{L}^{-1}, \epsilon,N}(\cdot,\cdot)$ by taking advantage of the eigendecomposition of the kernel, avoiding interpolation of eigenfunctions. We illustrate this idea using the $\epsilon$ kernel approximation $\nabla_1 K_{\mathscr{L}^{-1}, \epsilon}(\cdot,\cdot)$; its empirical counterpart follows directly. Set $\sigma_i = \frac{1 - \lambda_i}{\epsilon}$, and recall from \eqref{eigen_P} that $P_\epsilon(x,y) = \sum_{i = 1}^\infty \lambda_i \phi_i(x) \phi_i(y)$. Consider the following eigendecomposition of the kernel: 
\begin{align}
     &\mathcal \nabla_1 K_{\mathscr{L}^{-1}, \epsilon}(x^*,y^*) \nonumber\\
    = &\int_\mathcal Z \int_\mathcal W \left(\sum_{k = 1}^\infty \lambda_k \nabla \phi_k(x^*) \phi_k(w)\right)   
   \left(\sum_{j = 1}^\infty \lambda_j^{-1}\sigma_j^{-1} \lambda_j^{-1} \phi_j(w) \phi_j(z)\right)
    \left(\sum_{i = 1}^\infty \lambda_i\phi_i(z) \phi_i(y^*)\right) d\pi(w)\, d\pi(z) \nonumber\\
    = & \int_\mathcal Z \int_\mathcal W \nabla_1 P_\epsilon(x^*,w)  \left(\sum_{j = 1}^\infty \lambda_j^{-1}\sigma_j^{-1} \lambda_j^{-1} \phi_j(w) \phi_j(z)\right) 
    P_\epsilon (z,y^*) d\pi(w)\, d\pi(z) \label{kernel_int}.
\end{align}
where we have used the orthogonality of the eigenfunctions.

As noted previously, we use $\{z^i\}_{i = 1}^{N}$ to represent the training samples and use $\{x^i_t\}_{i = 1}^{M}$ to represent the generated samples at time $t$. Focusing on a single time step, we drop the dependence on $t$. Then the empirical approximation of \eqref{kernel_int} is as follows:
\begin{align}
     \nabla_1 K_{\mathscr{L}^{-1}, \epsilon,N}(x^i,x^j)
    = &\sum_{k_1 = 1}^{N}\sum_{k_2 = 1}^{N} \nabla_1 P_{\epsilon,N}(x^{i},z^{k_1})  \left(\sum_{k_3 = 1}^{N} \phi_{k_3}(z^{k_1}) \lambda_{k_3}^{-1}\sigma_{k_3}^{-1} \lambda_{k_3}^{-1}\phi_{k_3}(z^{k_2}) \right) P_{\epsilon,N}(z^{k_2},x^{j}). \label{grad_K_inv}
\end{align}
 In the matrix representation, the three matrices (from left to right) on the right-hand side above are of size $M \times N$, $N\times N$, and $N \times M$. The ingredients for computing the expression above are $P_{\epsilon,N}(\cdot,\cdot), \nabla_1 P_{\epsilon,N}(\cdot,\cdot), \lambda_i, \sigma_i, \phi_i$. Since $P_{\epsilon, N}$ is constructed using Gaussian kernels, its derivative with respect to the first argument $\nabla_1 P_{\epsilon,N}(\cdot,\cdot)$ can be computed in closed form (see Appendix~\ref{analytic_grad}). We therefore obtain an implementable algorithm.

\subsection{Algorithm for generative modeling}
\begin{algorithm}[tb]
   \caption{Diffusion map particle system (DMPS)}
   \label{alg:DMPS}
\begin{algorithmic}
    \STATE {\bfseries Input:} Training samples $\{z^i\}_{i = 1}^{N}\sim \pi$ and initial particles $\{x_0^i\}_{i = 1}^{M}$, tolerance $\mathrm{tol}$, bandwidth $\epsilon$, step size $h$
    \STATE {\bfseries Output:} $\{x_T^i\}_{i = 1}^{M}$
    \STATE Construct $P_{\epsilon, N}$ using $\{z^i\}_{i = 1}^{N}$ as in \eqref{finite_P_N}.
    \STATE Compute the eigenpairs $\{\lambda_i, \phi_i\}$ such that $P_{\epsilon,N}(x,y) = \sum_{i = 1}^N\lambda_i \phi_i(x) \phi_i(y)$, i.e., by performing singular value decomposition on the kernel matrix $P_{\epsilon,N}$.
    \WHILE{$\mathrm{tol}$ not met}
    \FOR{$i = 1,\cdots, M$}
    \STATE $x_{t+1}^{i} = x_{t}^i - \frac{h}{M} \sum_{j = 1}^{M} \nabla_1 K_{\mathscr{L}^{-1}, \epsilon, N} (x_{t}^i, x_{t}^j)$ as in \eqref{grad_K_inv}.
    \ENDFOR
    \ENDWHILE  
\end{algorithmic}
\end{algorithm}

Algorithm \ref{alg:DMPS} summarizes our proposed scheme, called a diffusion map particle system (DMPS).  We offer several comments. The classical analysis of diffusion maps requires the underlying distribution to have bounded support, but we find that this algorithm works well even when the support of $\pi$ is (in principle) unbounded. We suggest to initialize the samples $\{x_0^i\}_{i = 1}^{M}$ inside the support of $\pi$. Even though initializing samples outside the support would work because of the finite bandwidth $\epsilon$, starting the samples inside the support generally makes the algorithm more stable. In practice, we choose the bandwidth as $\epsilon = \mathrm{med}^2/(2\log N)$, following the heuristics proposed in \citet{svgd1}, where $\mathrm{med}$ denotes the median of the pairwise distances between training samples. \rev{A small $\epsilon$ results in particles not diffusing sufficiently, while a large $\epsilon$ causes particles to diffuse excessively; this heuristic seeks to balance such extremes. The tolerance terminating Algorithm \ref{alg:DMPS} is met if the Euclidean distance between successive positions of a particle, averaged over all particles, falls below a specified threshold for two consecutive iterations.}

\section{Convergence analysis}
We comment on the convergence rate of our scheme at the population level. From \eqref{conv_rate}, we see that if the kernel is exact, then the rate of change of the KL divergence is $-\mathbb E_\pi \left[ \frac{d\mu_t}{d\pi} \mathscr L \mathcal K_\pi \frac{d\mu_t}{d\pi}\right]$. If we replace $\mathcal K_\pi$ by its kernel approximation $L_{\epsilon}^{-1}$, then the resulting rate of change is
\begin{align*}
    \partial_t \DKL(\hat \mu_t||\pi) = -\mathbb E_\pi \left[ \frac{d\hat \mu_t}{d\pi} \mathscr L L_{\epsilon}^{-1} \frac{d\hat \mu_t}{d\pi}\right],
\end{align*}
where $\hat \mu_t$ is the distribution at time $t$ obtained from the following evolution
\begin{align*}
    \dot{x} = \int \nabla_1 K_{\mathscr L^{-1}, \epsilon}(x,y) d \hat  \mu_t(y).
\end{align*} 
Classical results from diffusion map literature \citep{Hein_rate, Singer_rate} reveal that the bias $\left| \mathscr L f(x)- L_\epsilon f(x)\right| \sim \mathcal O(\epsilon)$ if data lie on a compact manifold. Using the same assumptions, we state the following theorem. 
\begin{theorem}\label{thm:convergence}
    Suppose the target distribution $\pi$ is supported on a compact manifold. Let  $\mu_0$ be the initial distribution of the particles and $\hat \mu_t$ be the distribution of the generated process at time $t$, and assume that $\frac{d\hat \mu_t}{d\pi}$ \rev{exists} and is twice differentiable for all $t$. Then we have 
    \begin{align*}
    \rev{
        \DKL(\hat \mu_t||\pi) \leq \bigl (\mathcal O(\epsilon) + \DKL(\mu_0||\pi)\bigr ) e^{-t} + \mathcal O(\epsilon).
        }
    \end{align*}
\end{theorem}
\begin{proof}
From \rev{the} diffusion map approximation, we have $\left| \mathscr L \frac{d\hat \mu_t}{d\pi}(x)- L_\epsilon \frac{d\hat \mu_t}{d\pi}(x)\right| \sim \mathcal O(\epsilon)$. Then obtain that, 
\begin{align*}
    \mathscr L \frac{d\hat \mu_t}{d\pi}(x) =\left(   L_\epsilon + \mathcal O(\epsilon)\right)\frac{d\hat \mu_t}{d\pi}(x)
\end{align*}
by factoring out $\frac{d\hat \mu_t}{d\pi}(x)$. Then using Neumann series or binomial expansion
\begin{align*}
     \mathscr L^{-1} \frac{d\hat \mu_t}{d\pi}(x) &=\left( L_\epsilon + \mathcal O(\epsilon) \right)^{-1}\frac{d\hat \mu_t}{d\pi}(x) \\
     &= \left(  L_\epsilon^{-1} - \mathcal O(\epsilon) L_\epsilon^{-2} +  \mathcal O(\epsilon)^2\right)\frac{d\hat \mu_t}{d\pi}(x) \\
     &= \left( L_\epsilon^{-1} + \mathcal O(\epsilon)  \right)\frac{d\hat \mu_t}{d\pi}(x).
\end{align*}
This follows from the fact that the inverse $L_\epsilon^{-1}$ is bounded. We then have 
\begin{align*}
    \mathscr L^{-1} = L_\epsilon^{-1}  + \mathcal O(\epsilon).
\end{align*}
We can then write 
\begin{align*}
    &\mathbb E_\pi \left[ \frac{d\hat \mu_t}{d\pi} \mathscr L L^{-1}_\epsilon \frac{d\hat \mu_t}{d\pi}\right]\\
    = &\mathbb E_\pi \left[ \frac{d\hat \mu_t}{d\pi} \mathscr L \left(\mathscr L^{-1} \frac{d\hat \mu_t}{d\pi} + \mathcal O(\epsilon) \right)\right]\\
    = &\mathbb E_\pi \left[ \left(\frac{d\hat \mu_t}{d\pi}-1 \right)^2\right] + \mathcal O(\epsilon),
\end{align*}
and consequently by \eqref{chi_rate},
\begin{align*}
     \partial_t \DKL(\hat \mu_t||\pi) =  -\chi^2(\hat \mu_t|| \pi) + \mathcal O (\epsilon).
\end{align*}
Then using the results (Theorem 1) in \citet{Chewi2020SVGD} and Gronwall's inequality, we have

\begin{align*}
\rev{
    \DKL(\hat \mu_t||\pi) \leq \left(\mathcal O(\epsilon) + \DKL(\mu_0||\pi)\right) e^{-t} + \mathcal O(\epsilon).
    }
\end{align*}
\end{proof}

\rev{We see that for a fixed $\epsilon$, as $t\rightarrow \infty$, the KL divergence between the target distribution and the initial distribution vanishes exponentially, and the error is dominated by the error arising from the diffusion map approximation.}

\section{Related work}\label{sec:relatedwork}

Several recent studies have explicitly considered the problem of generative modeling on manifolds. \citet{RSGMBortoli} construct a score-based generative model on a Riemannian manifold by coupling an Euler–Maruyama discretization on the tangent space with the exponential map to move along geodesics. Unlike the present method, however, this approach requires the underlying manifold structure to be prescribed. 
\citet{SGMFlows} use ideas from GANs, VAEs, and normalizing flows to learn the manifold and the density on that manifold simultaneously;  \citet{reg_flow_Caterini} propose a ``rectangular'' normalizing flow for for a similar setting, where the underlying manifold is unknown. These neural network based methods are quite complex and require significant tuning, with performance depending on the richness of the approximation families (e.g., width and depth of the networks, architectural choices in the normalizing flows). Our method is comparatively much simpler and nonparametric: the expressivity of the approximation grows with the sample size, and the only parameter that requires tuning is the kernel bandwidth. The role of manifold structure in score-based generative modeling has also been explored recently by \citet{jakiwSGM}, who shows that if the learned score is sufficiently accurate, samples generated by the SGM lie exactly on the underlying manifold. Our numerical experiments below will show that SGMs can indeed successfully identify the underlying manifold structure, but are comparatively more expensive and less accurate. 

\rev{A similar idea has been used in~\citet{gottwald2024stablegenerativemodelingusing}, where diffusion maps are employed to approximate the score function. While both approaches utilize diffusion maps, the key difference is that~\citet{gottwald2024stablegenerativemodelingusing} uses diffusion maps specifically to approximate the score function, and then adopts Langevin dynamics for \textit{stochastic} generative modeling. In contrast, our approach directly uses the diffusion map approximation of the generator of the Langevin diffusion process, without specializing this approximation to the score; this approach enables the use of LAWGD and results in \textit{deterministic} generative modeling.}

\section{Numerical experiments} \label{sec:num_res}
In this section, we study four numerical examples, exploring the performance of the algorithm on connected and disconnected domains and on manifolds. To benchmark the performance of DMPS, we compare it with: (i) SVGD and (ii) ULA, where the score $\nabla V$ required by both algorithms is replaced by its empirical approximation using the diffusion map, as well as with (iii) the score-based generative model (SGM) of \citet{diffusion2}. We implement SGM using a lightweight notebook from \citet{jakiw_blog,jakiwSGM}. 
To make (i) and (ii) more precise: recall that the a diffusion map can be used as a tool for approximating the Langevin generator from samples. That is, $\mathscr L f  = \nabla^2 f - \langle \nabla V, \nabla f \rangle$. Then note that by letting $f $ be the identity, i.e., $f(x) = x$, we have that $\mathscr L (x) = \nabla \log \pi$. Therefore, we can use samples to approximate the gradient of the potential. 
For DMPS and SVGD, we run the algorithm until a prescribed tolerance is met and we run ULA and SGM for a fixed number of iterations. 

To evaluate the quality of samples generated with each method, we compute the regularized optimal transport (OT) distance between the generated samples and reference samples from the target distribution. We compute this distance using the
Sinkhorn--Knopp algorithm \citep{sinkhorn, sinkhorn_knopp}. The cost matrix is set to be the pairwise distance between the reference samples and generated particles, and each sample is assigned equal weight marginally. The number of reference samples is chosen to be large to mitigate error in the OT distance resulting from discretization of the target: $20000$ for the first three examples and $50000$ for the last one due to its higher dimension. The entropic regularization penalty $1/\lambda$ is set to be $O(10^{-2})$ in Mickey mouse, two moons, the arc, and hyper-semisphere examples, and $O(10^{-3})$ in the high energy physics example. 
For Mickey mouse, two moons, the arc, and hyper-semisphere examples, each experiment is repeated $10$ times for reproducibility; this replication involves sampling new training data and repeating all steps of each algorithm. For the high energy physics example, the experiment is conducted for $30$ different physical particles.
For the the Mickey mouse, two moons, the arc, the high energy physics examples, the number of generated particles is varied over $\{100, 300, 900, 2700$\}; for the hyper-semisphere example, it is fixed to $300$. 

\subsection{Mickey mouse: two-dimensional connected domain}
In this example, the target distribution is uniform over a compactly supported Mickey mouse-shaped domain. The generative process is initiated uniformly inside a circle. Results are obtained with both $1000$ and $2000$ training samples. In Figure \ref{mickey_sample}, we show the initial particles, the generated particles and the target distribution. Both methods capture the shape relatively well. However, particles generated from SVGD move out of the domain, while most of the particles generated using DMPS stay inside. In some cases, the SVGD-generated particles exhibit a non-uniform pattern; see Figure \ref{mickey_wierdsvgd}. Figure \ref{mickey_err} shows quantitative comparisons of the error. 
\begin{figure}
\centering  
\subfigure{\includegraphics[width=0.48\linewidth]{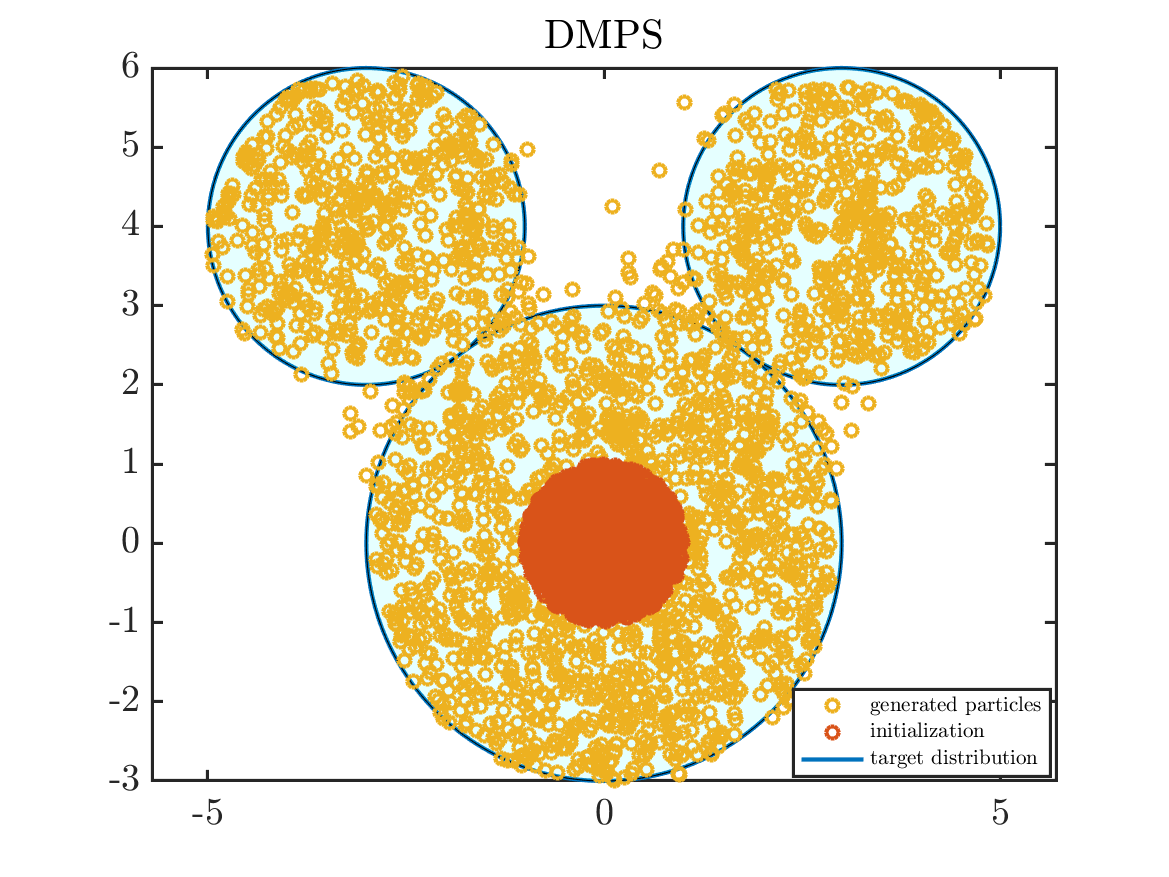}}
\subfigure{\includegraphics[width=0.48\linewidth]{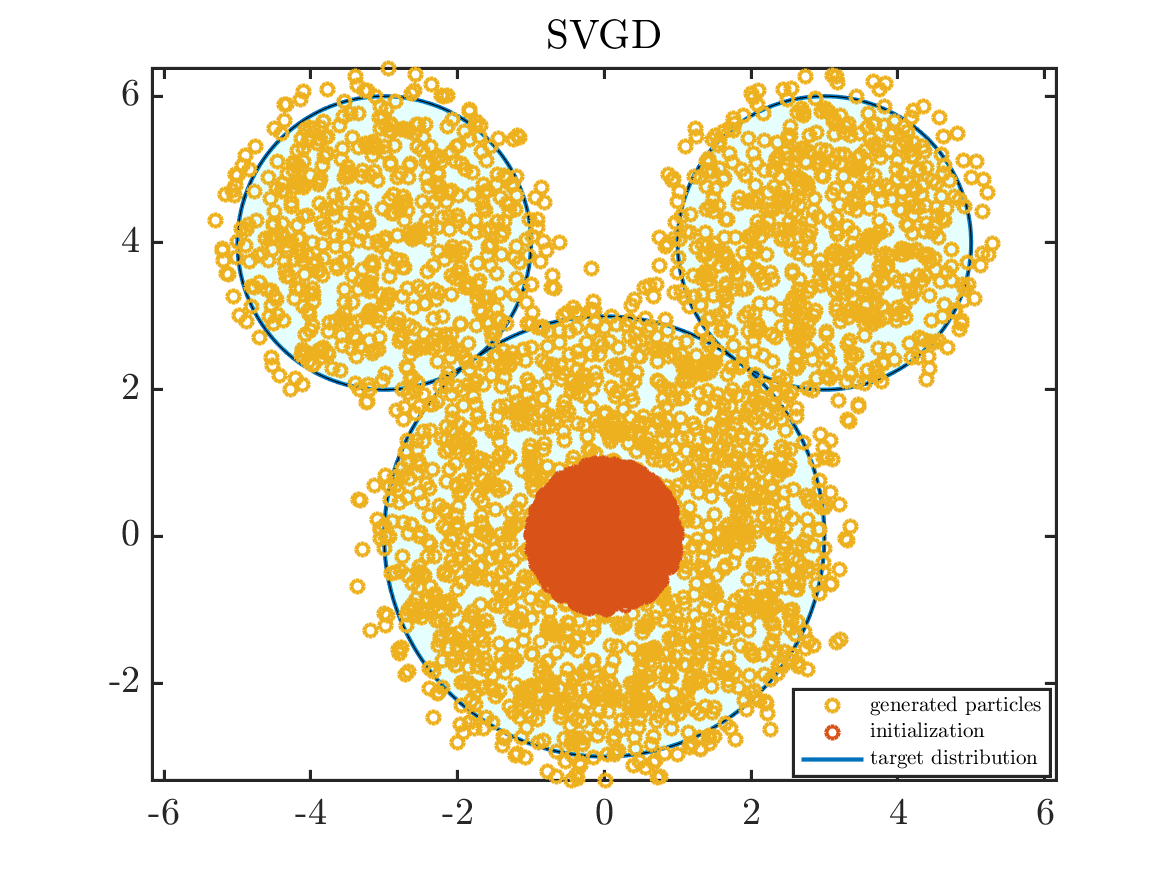}}
\caption{Mickey mouse: 2700 generated particles using DPMS and SVGD, with 2000 training samples}
\label{mickey_sample}
\end{figure}

\begin{figure}
\begin{center}
      \includegraphics[width=0.48\linewidth]{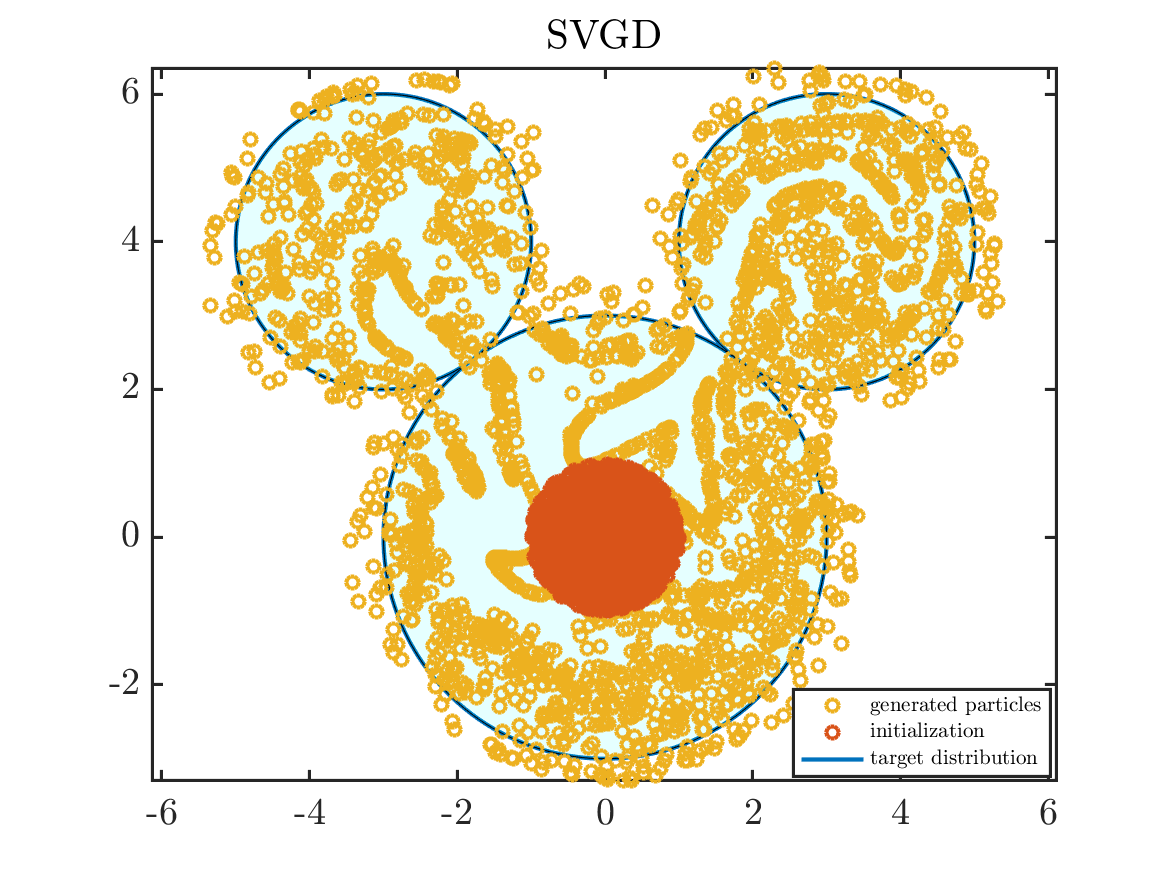}
      \caption{Mickey mouse: an instance of running the SVGD generative model shows strange non-uniform pattern with 1000 training samples and 2700 generated particles}
      \label{mickey_wierdsvgd}
\end{center}
\end{figure}

\begin{figure}
\begin{center}
      \includegraphics[height=0.5\linewidth]{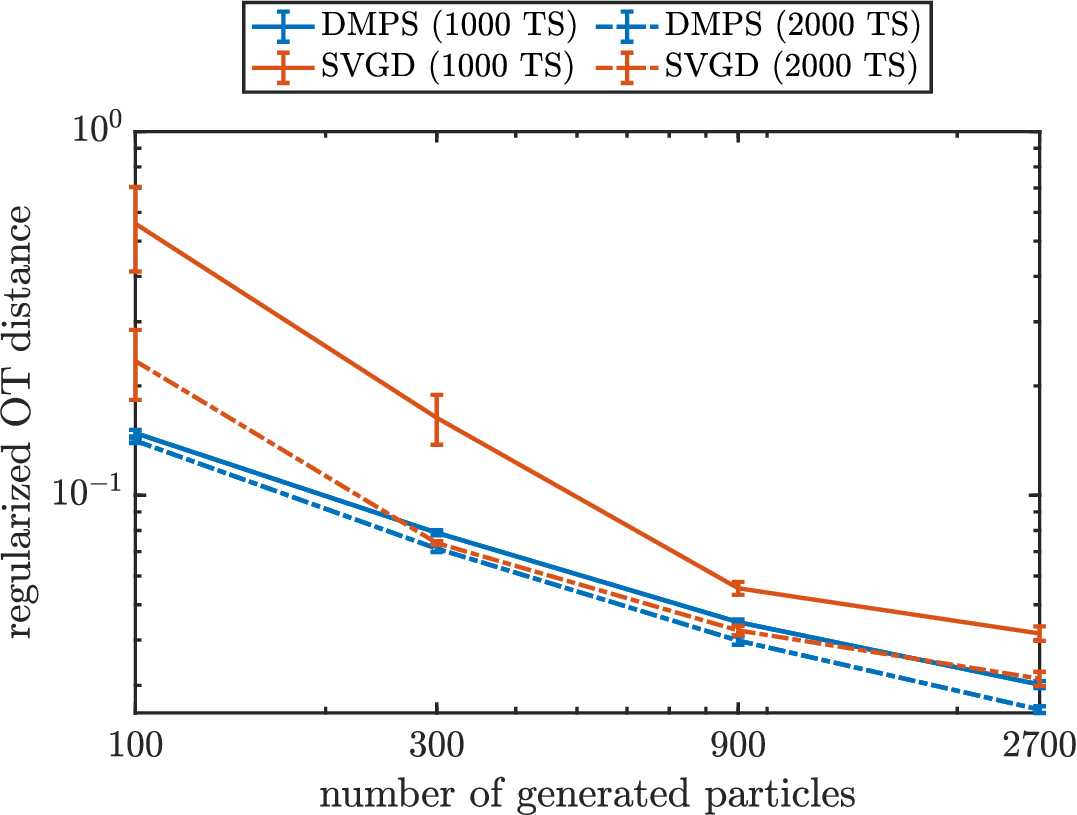}
      \caption{Mickey mouse: error comparison between DMPS, SVGD}
      \label{mickey_err}
\end{center}
\end{figure}

\subsection{Two moons: two-dimensional disconnected domain}
In this example, the target distribution is uniform and compactly supported on a two-moon-shaped domain. In contrast with the previous example, the domain is disconnected. Though the underlying distribution has zero density outside the support, the finite kernel bandwidth enables the methods to be implementable in this case. 
Results are obtained with $500$ and $1000$ training samples. We show the initial particles, target distribution, and particles generated with DMPS, SVGD, and ULA in Figure \ref{twomoons_sample} and the regularized OT distance in Figure \ref{twomoons_err}. As we can see in Figure \ref{twomoons_sample}, SVGD does not explore the very end of the domain and ULA has  mny samples that diffuse out of the support. The error plot (Figure \ref{twomoons_err}) shows that DMPS enjoys the smallest error in terms of OT distance, and that this the error decreases with more generated particles. While ULA shows a similar convergence (with larger values of error), the error of SVGD fluctuates as more particles are included.

\begin{figure}[!t]
\centering
\begin{tabular}{ccc}
\resizebox{0.33\textwidth}{!}{\includegraphics{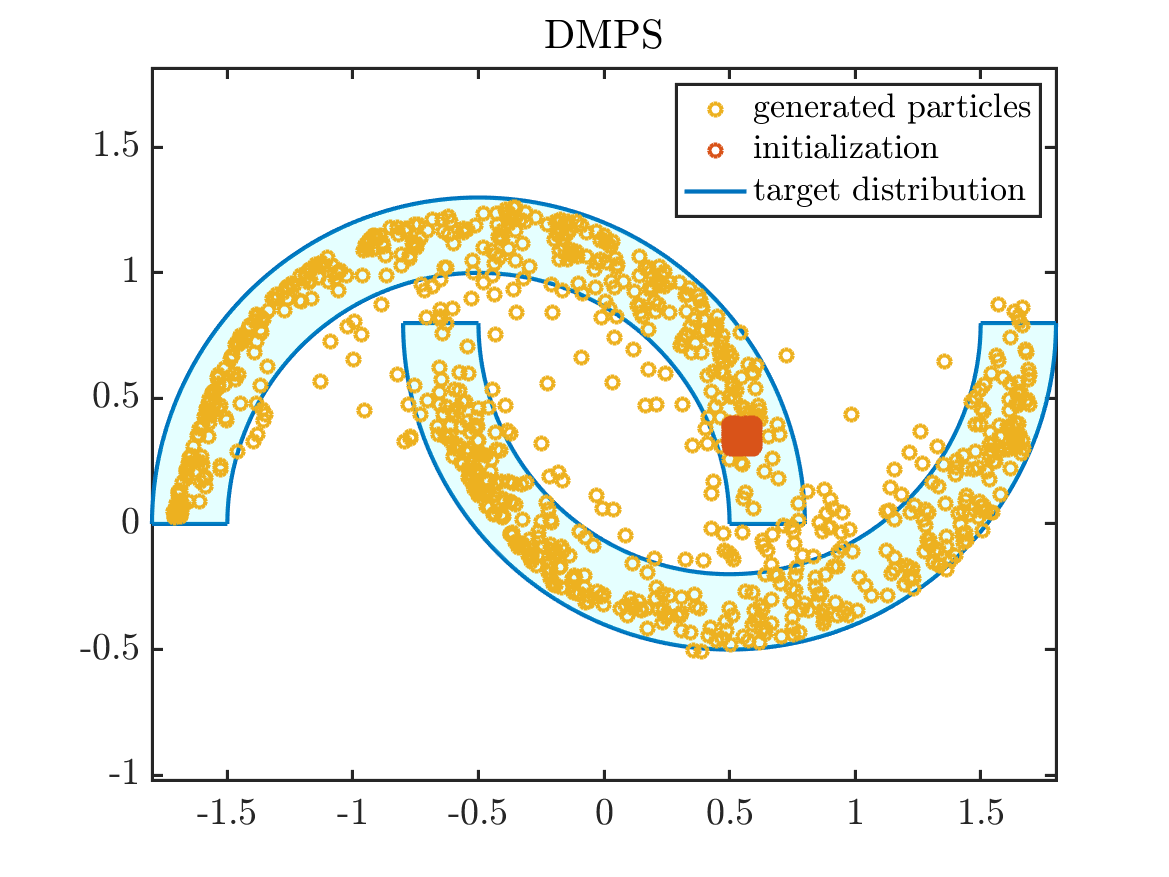}}
&
\hspace*{-0.7cm}\resizebox{0.33\textwidth}{!}{\includegraphics{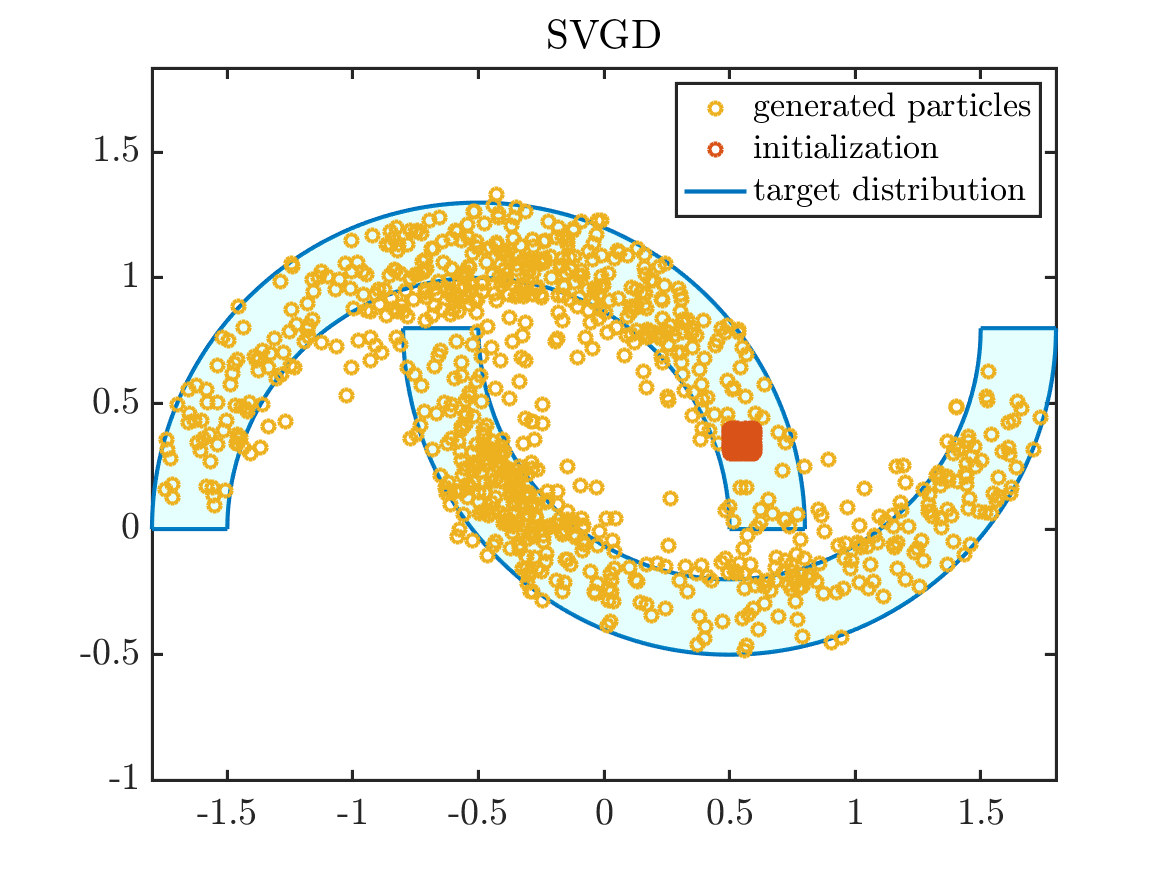}}
&
\hspace*{-0.7cm}\resizebox{0.33\textwidth}{!}{\includegraphics{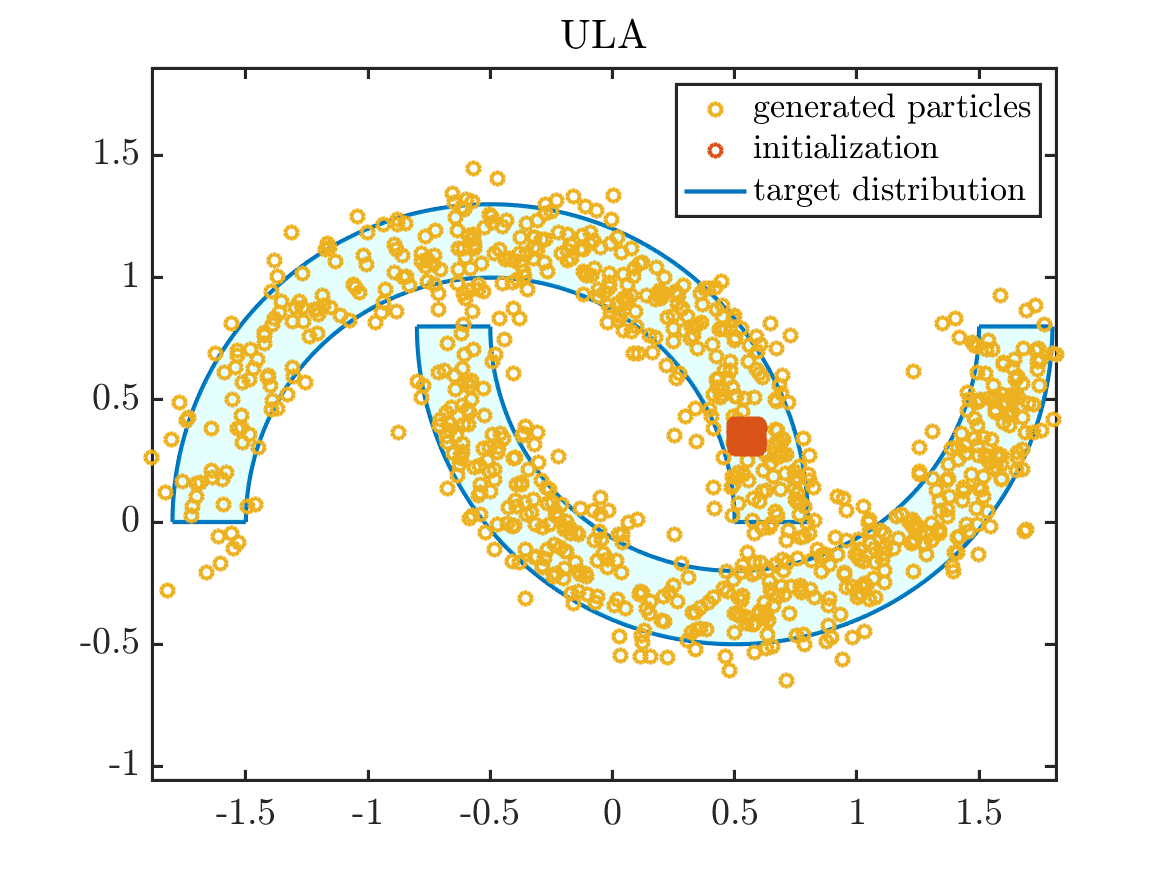}}\\
\resizebox{0.33\textwidth}{!}{\includegraphics{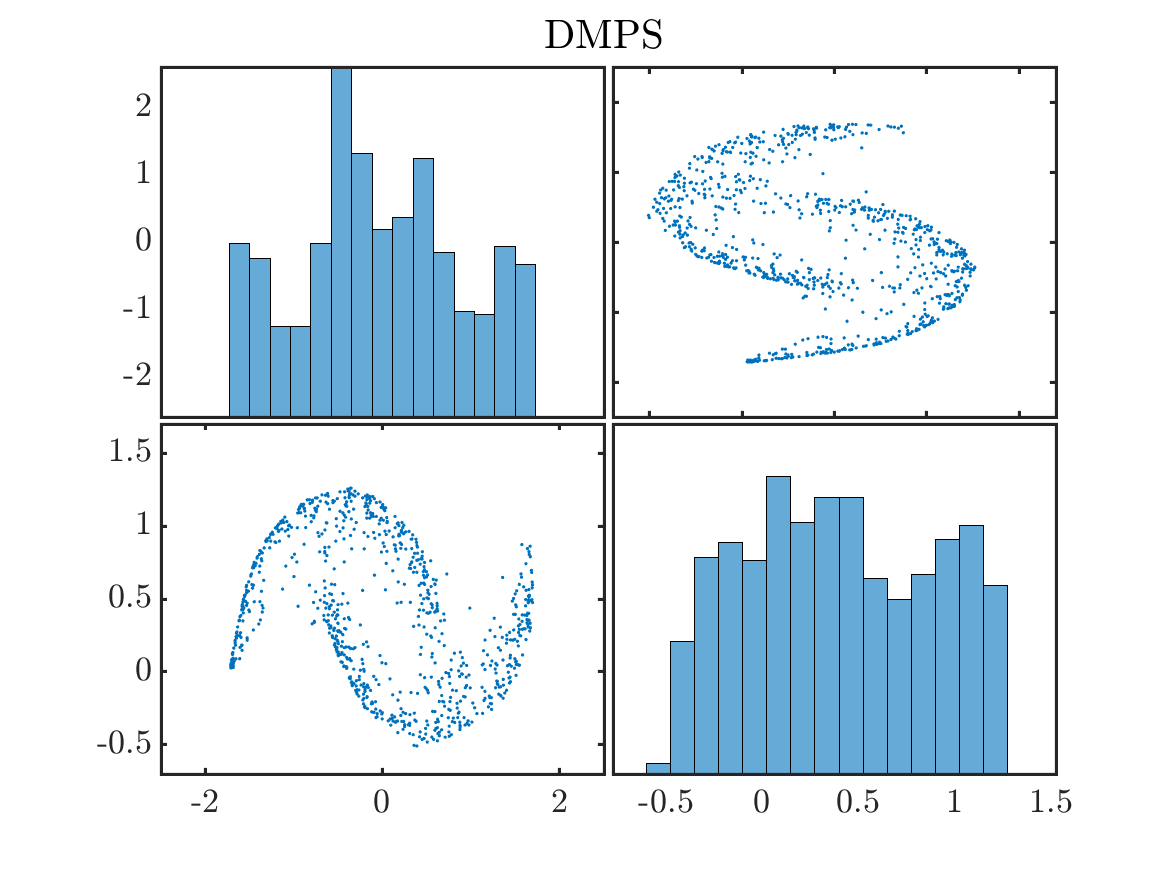}}
&
\hspace*{-0.7cm}\resizebox{0.33\textwidth}{!}{\includegraphics{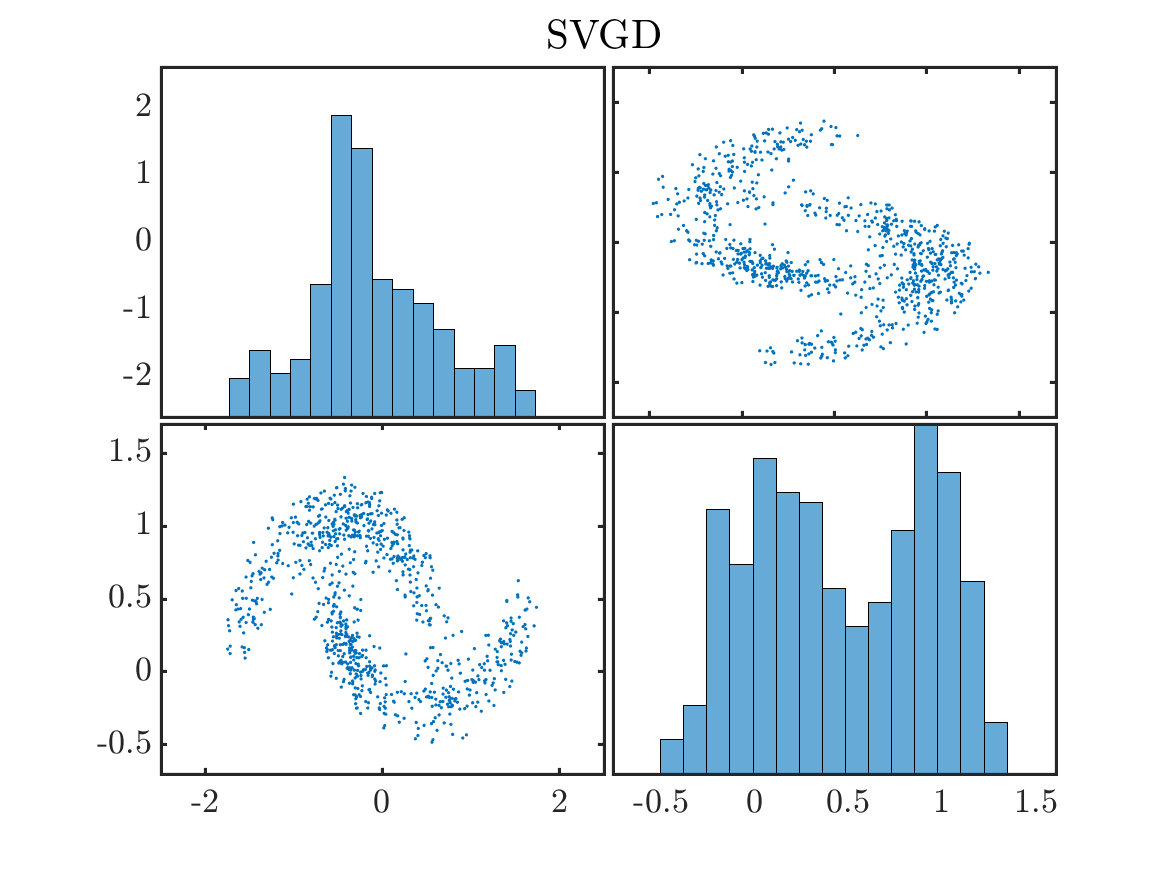}}
&
\hspace*{-0.7cm}\resizebox{0.33\textwidth}{!}{\includegraphics{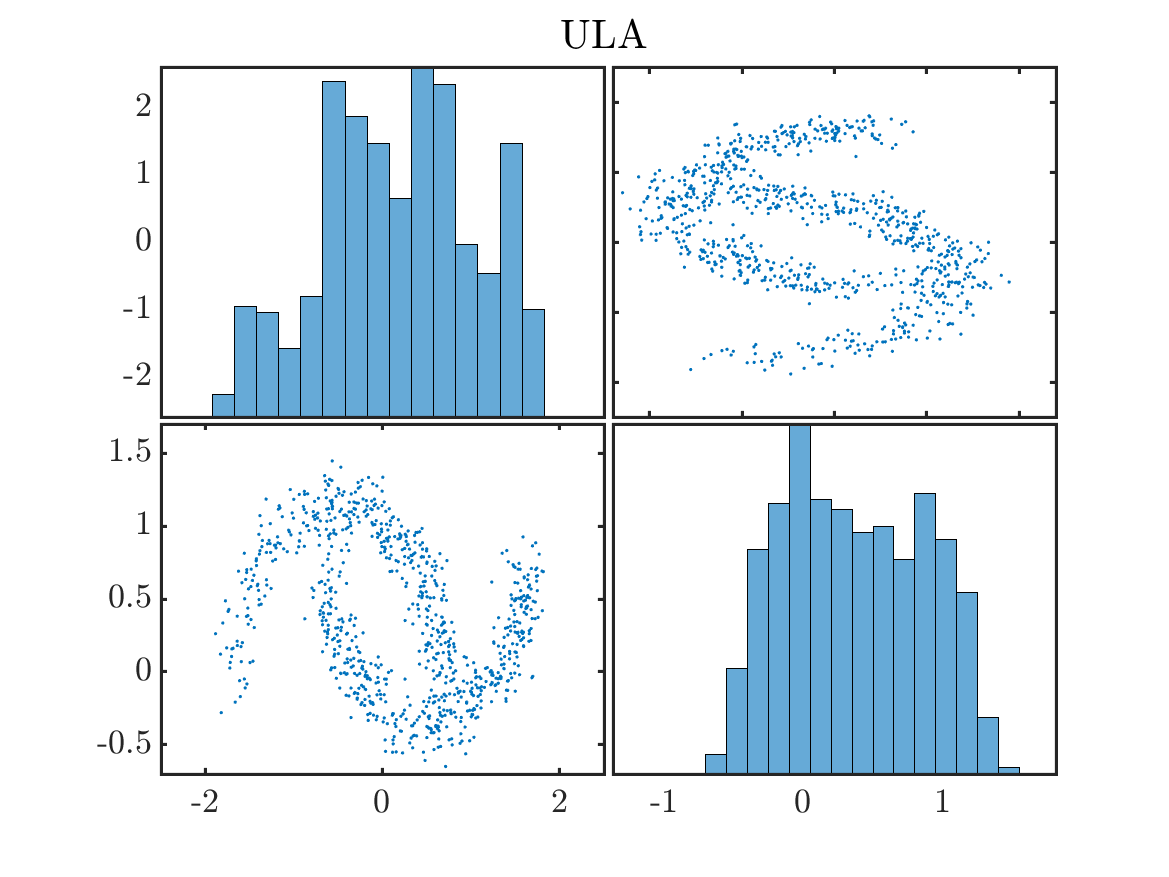}}\\
\end{tabular}
\caption{Two moons: 900 generated particles from DMPS, SVGD, and ULA with 500 training samples}
\label{twomoons_sample}
\end{figure}

\begin{figure}
\begin{center}
      \includegraphics[height=0.5\linewidth]{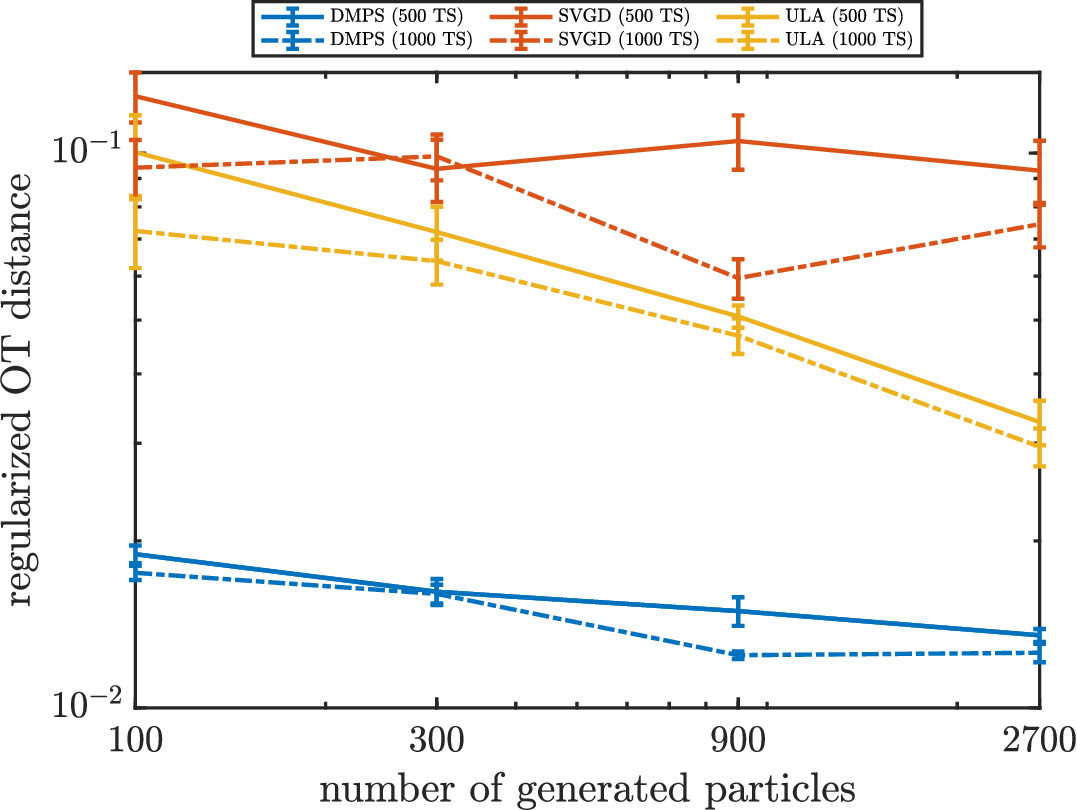}
      \caption{Two moons: error comparison between DMPS, SVGD, and ULA. Solid lines use 500 training samples, dashed lines use 1000.}
      \label{twomoons_err}
\end{center}
\end{figure}

\subsection{The arc: one-dimensional manifold embedded in a three-dimensional space}
We now consider an example where the data lie on a manifold, in this case an arc of radius $1$ embedded in $\mathbb{R}^3$. Training data are drawn uniformly from the arc and perturbed in the radial direction only, with $U(0,10^{-2})$ noise.
Results are obtained with both 100 and 1000 training samples. The initial particles and the target distribution are shown in Figure \ref{3d_rotation_sample} (left). We then run DMPS, SVGD, ULA, and SGM for each batch of training and initialization samples, visualizing an instance in Figure \ref{3d_rotation_sample}. Particles generated by the two deterministic methods, DMPS and SVGD, lie only on the two-dimensional plane of the training data, but the particles generated using SVGD do not fully explore the target distribution. Particles generated by the two stochastic methods, ULA and SGM, span the full three-dimensional space due to the added noise. Errors are plotted in Figure \ref{3d_rotation_err}, for both choices of training set size. We see that DMPS exhibits the smallest errors, and that this error decreases as we increase the number of generated particles. SGM shows a similar trend, but with larger errors. The performance of ULA does not seem to improve after using more training samples, which might be due to finite discretization timestep (although it was chosen small relative to the width of the target, $h=5 \times 10^{-4}$).
SVGD gives the largest errors, which do not seem to decrease with more particles.

\begin{figure}\centering
\captionsetup[subfigure]{labelformat=empty}
\subfigure[Initialization and target distribution]{\label{}\includegraphics[width=.5\linewidth]{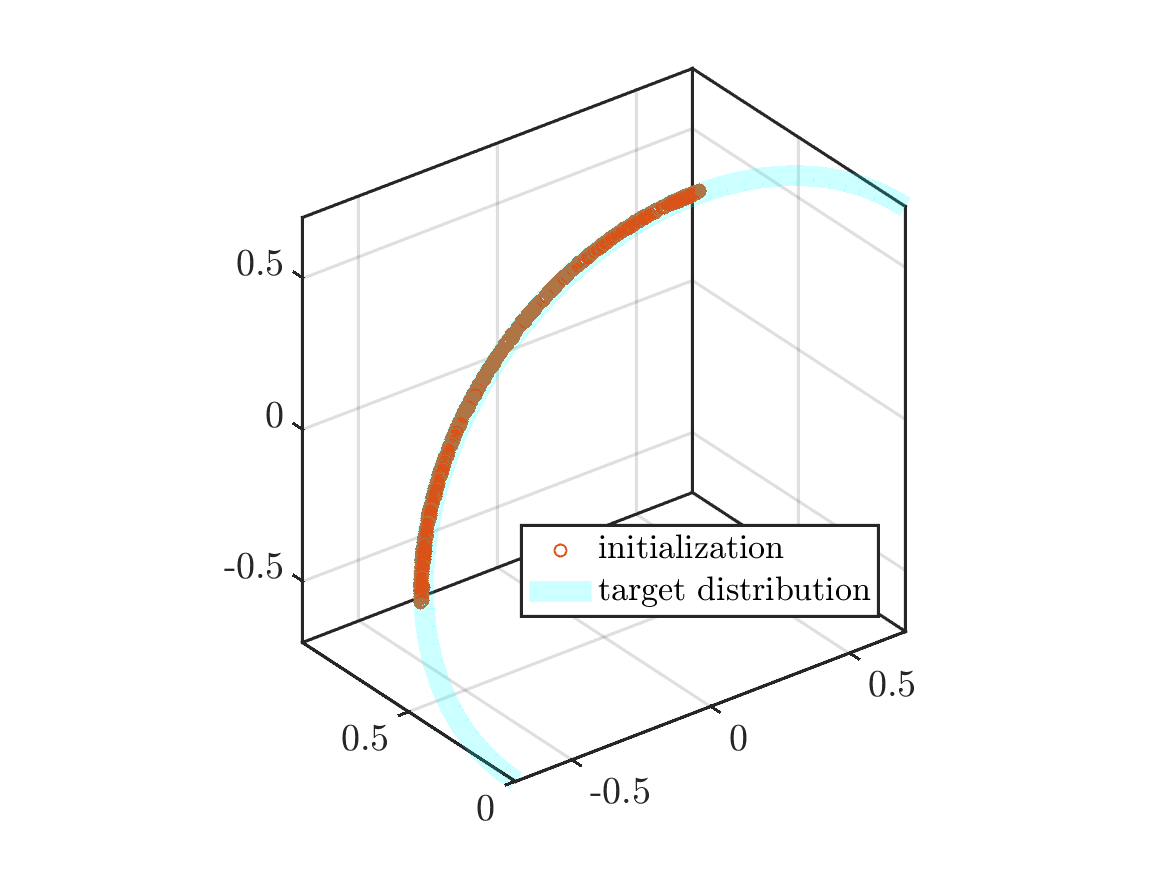}}\hfill
\subfigure[Generated particles]{\label{}\includegraphics[width=.5\linewidth]{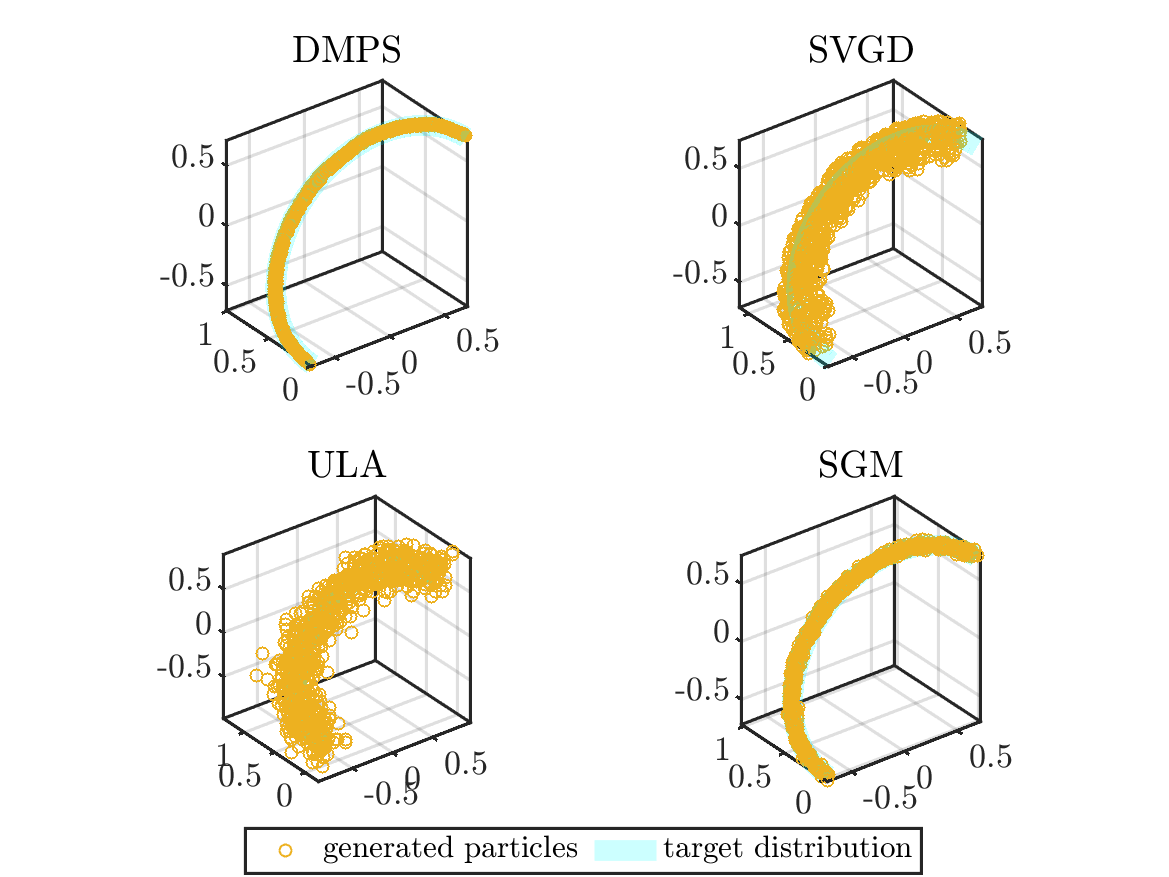}}\par 
\subfigure[Marginal distributions]{\label{}\includegraphics[width=1\linewidth]{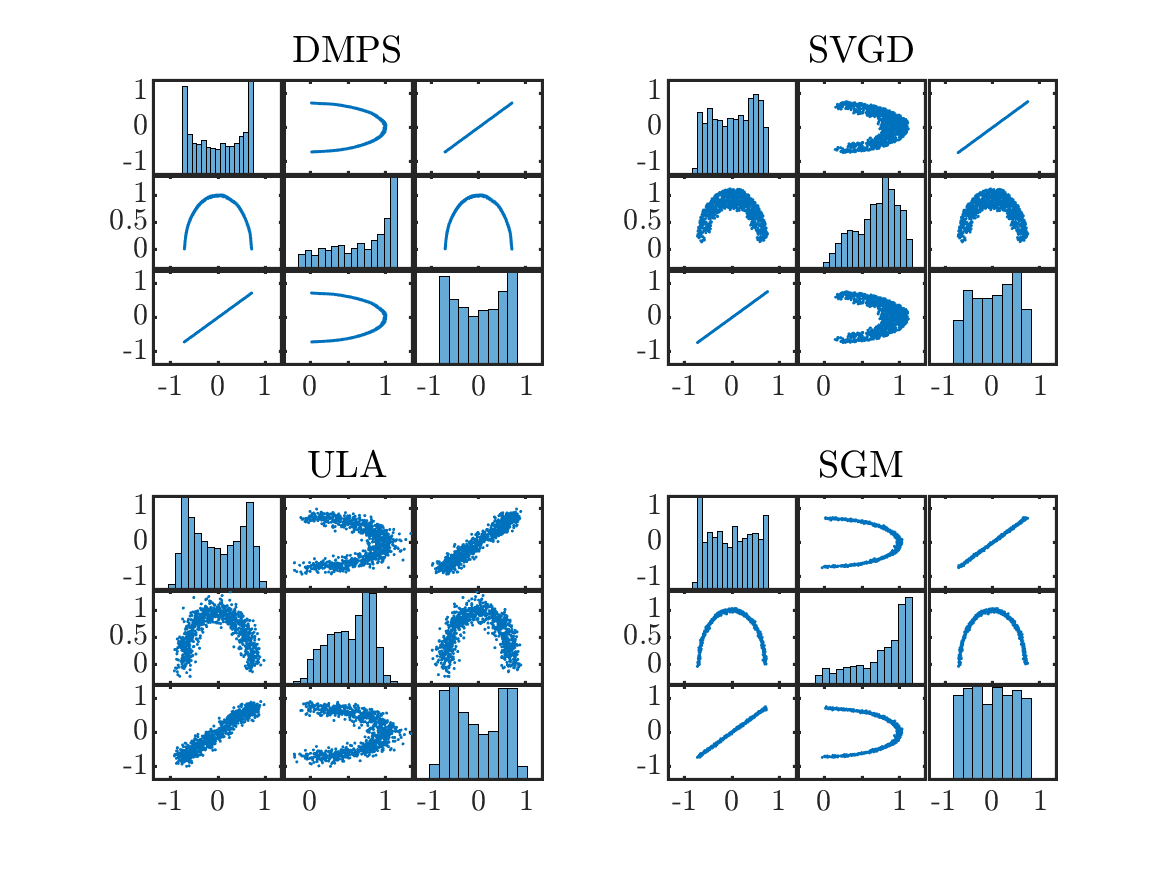}}
\caption{The arc: 900 generated particles using DMPS, SVGD, ULA, and SGM with 1000 training samples.}
\label{3d_rotation_sample}
\end{figure}

\begin{figure}
\begin{center}
      \includegraphics[height=0.5\linewidth]{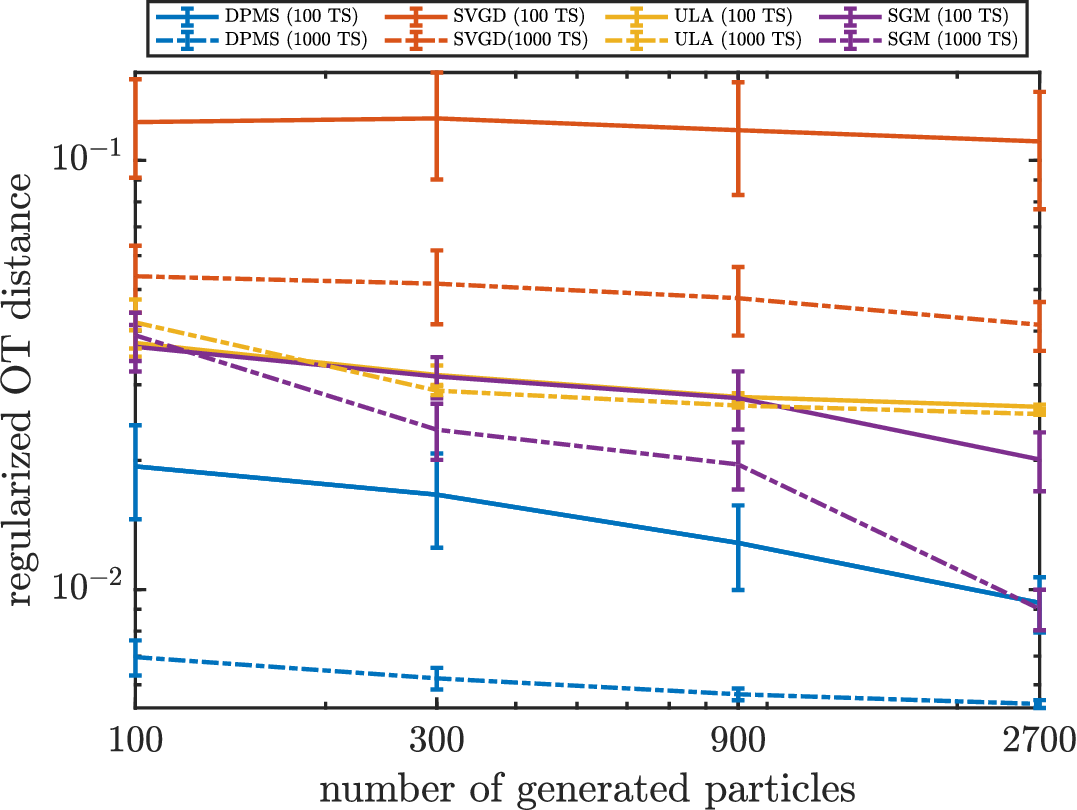}
      \caption{The arc: error comparison between DMPS, SVGD, ULA, and SGM. Solid lines use 100 training samples; dashed lines use 1000.}
      \label{3d_rotation_err}
\end{center}
\end{figure}

\subsection{Hyper-semisphere: 2 to 14-dimensional manifolds}\label{sec:semisphere}
We finally study an example where data are uniformly sampled on a half-sphere embedded in ambient dimensions $d \in \{3, 6, 9, 12, 15\}$; in each case, the manifold is thus of dimension $d-1$. 
For this problem, the number of training samples and the number of generated particles are fixed to $1000$ and $300$, respectively. A visualization for $d=3$ can be seen in Figure \ref{ndsphere_sample}. We show the error (in OT distance), and the standard error of the mean error over 10 trials, in Table \ref{err_table}. For all dimensions, DMPS enjoys the smallest error and the smallest standard error, followed by SGM and ULA, which also produce relatively small errors and stable results. SVGD has the largest error and does not produce stable results (large variability over the 10 trials). 
\begin{figure}
\begin{center}
      \includegraphics[height=0.5\linewidth]{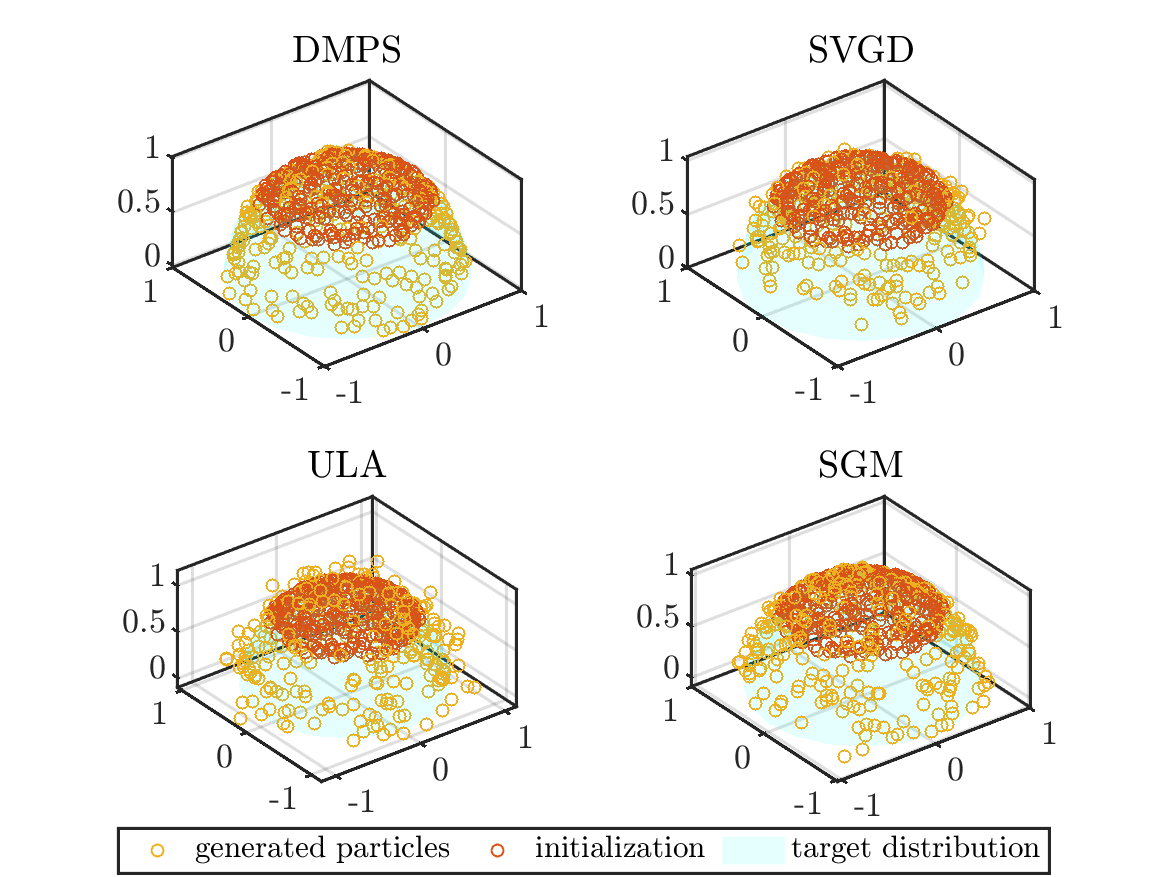}
      \caption{Hyper-semisphere: 300 generated samples using DMPS, SVGD, ULA and SGM in three dimensions with 1000 training samples.}
      \label{ndsphere_sample}
\end{center}
\end{figure}

\begin{table}
\centering
\begin{tabular}{ ccccc } 
\hline
& DMPS & SVGD & ULA &  SGM\\
\hline
$d$ = 3&\textbf{0.018} $\pm$ 0.0003&0.146 $\pm$ 0.0229&0.033 $\pm$ 0.0006&0.032 $\pm$ 0.0024\\
$d$ = 6&\textbf{0.142} $\pm$ 0.0003&0.267 $\pm$ 0.0171&0.185 $\pm$ 0.0012&0.170 $\pm$ 0.0022\\
$d$ = 9&\textbf{0.303} $\pm$ 0.0003&0.361 $\pm$ 0.0077&0.378 $\pm$ 0.0011&0.348 $\pm$ 0.0025\\
$d$ = 12&\textbf{0.441} $\pm$ 0.0002&0.811 $\pm$ 0.0783&0.555 $\pm$ 0.0018&0.496 $\pm$ 0.0041\\
$d$ = 15&\textbf{0.564} $\pm$ 0.0009&0.986 $\pm$ 0.0055&0.713 $\pm$ 0.0025&0.608 $\pm$ 0.0024\\
\hline
\end{tabular}
\caption{\label{err_table} Hyper-semisphere: error comparison ($\pm$ standard error) between DMPS, SVGD, ULA, and SGM.}
\end{table}

\subsection{High energy physics: gluon jet dataset}


We now study a real-world example from high energy physics, where the goal is to generate relative angular coordinates $\eta^{\textrm{rel}}, \phi^{\textrm{rel}}$ and relative transverse momenta $p_T^{\textrm{rel}}$ of elementary physical particles produced in a gluon jet. Details on the dataset are in \citet{kansal2021particle}. The dataset is of dimension $177252 \times 30 \times 4$, meaning that there are $177252$ jets, $30$ physical particles per jet and each particle is characterized by four distinct features: $\eta^{\textrm{rel}}$, $\phi^{\textrm{rel}}$, $p_T^{\textrm{rel}}$, and a binary mask. Since the value of the last feature is either $0$ or $1$, we only use the first three features for the propose of generative modeling. Therefore, for each physical particle, there are $177252$ available samples, and each sample is of $3$ dimensions. We then train a generative model for each of the $30$ physical particles. To train each model, we normalize the training data so that they have mean zero and marginal variances of one. We use $1000$ samples for training, which are drawn randomly from the full set of $177252$ samples, and initialize 100, 300, 900, or 2700 samples from $U(-1,1)^3$ for the generative process. We also compare SVGD, ULA and SGM with the same setup: 1000 training samples and increasing numbers of generated samples. Errors ($\pm$ standard error) for all four methods are shown in Table \ref{err_table_gluon}. These results are averaged over all $30$ different physical particles. We see that DMPS consistently outperforms the other methods across all cases in this problem.
In Figure~\ref{gluon_sample}, we also show the marginal distributions of 2700 generated samples (in red) using different methods and plot target samples (in blue) as a reference, for the first physical particle. As we can see, samples generated using DMPS, ULA, and SGM resemble those from the target distribution. Samples generated using SVGD, however, qualitatively fail to capture the target distribution, consistent with the larger errors in Table \ref{err_table_gluon}.

\begin{figure}
\centering 
\subfigure{\includegraphics[width=0.49\linewidth]{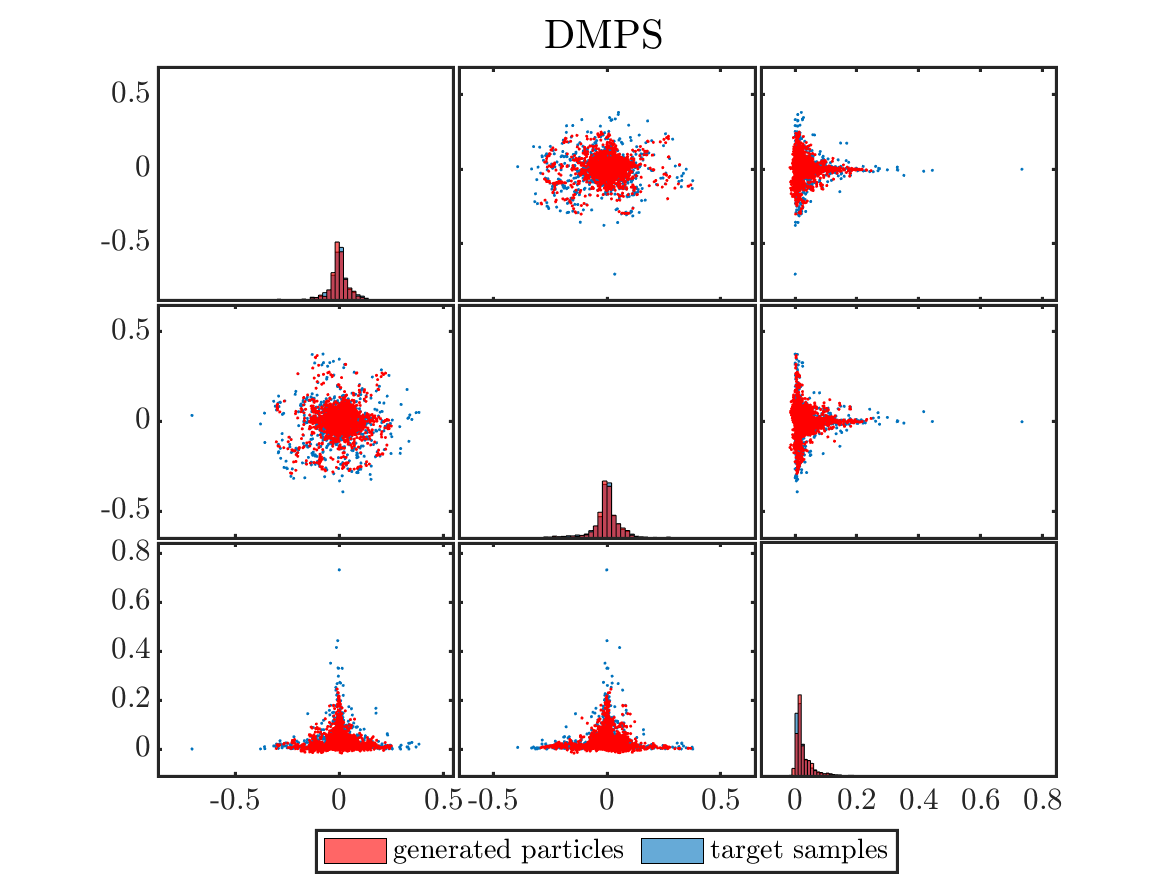}}
\subfigure{\includegraphics[width=0.49\linewidth]{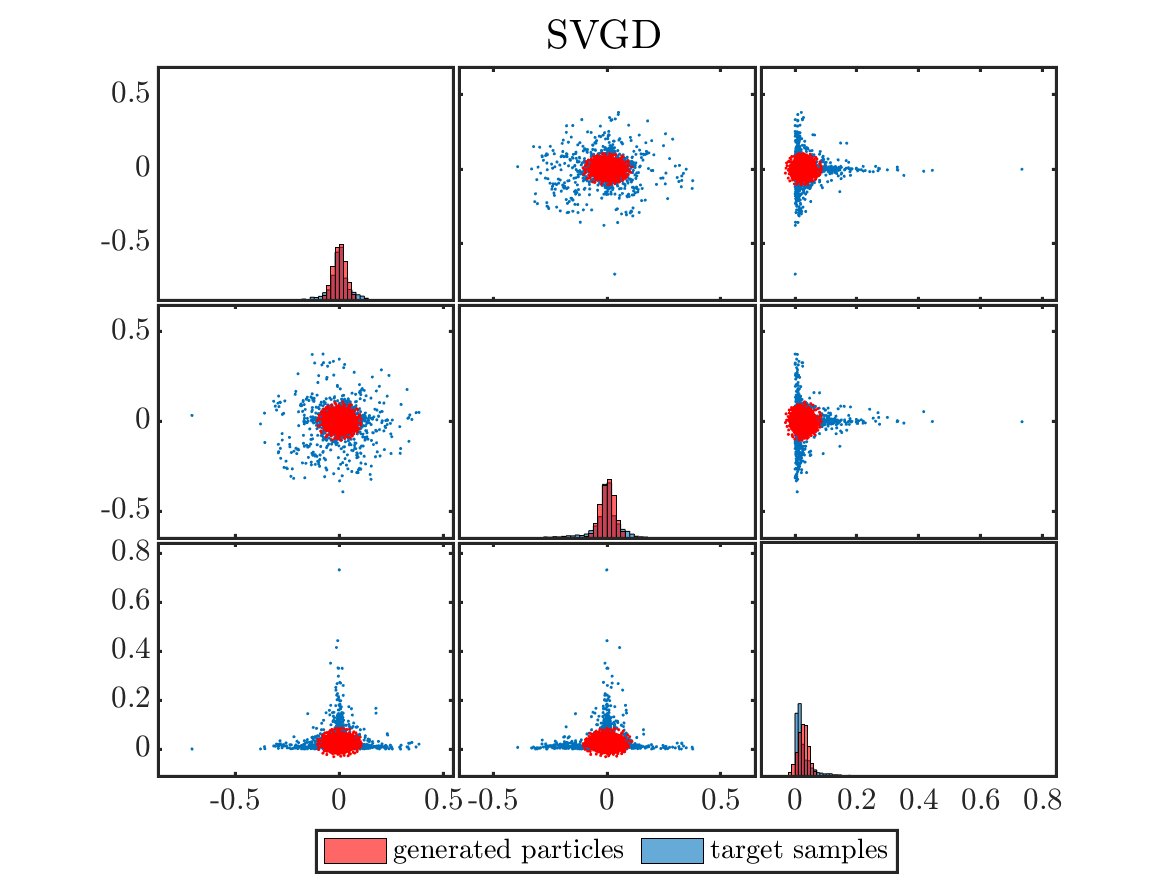}}
\subfigure{\includegraphics[width=0.49\linewidth]{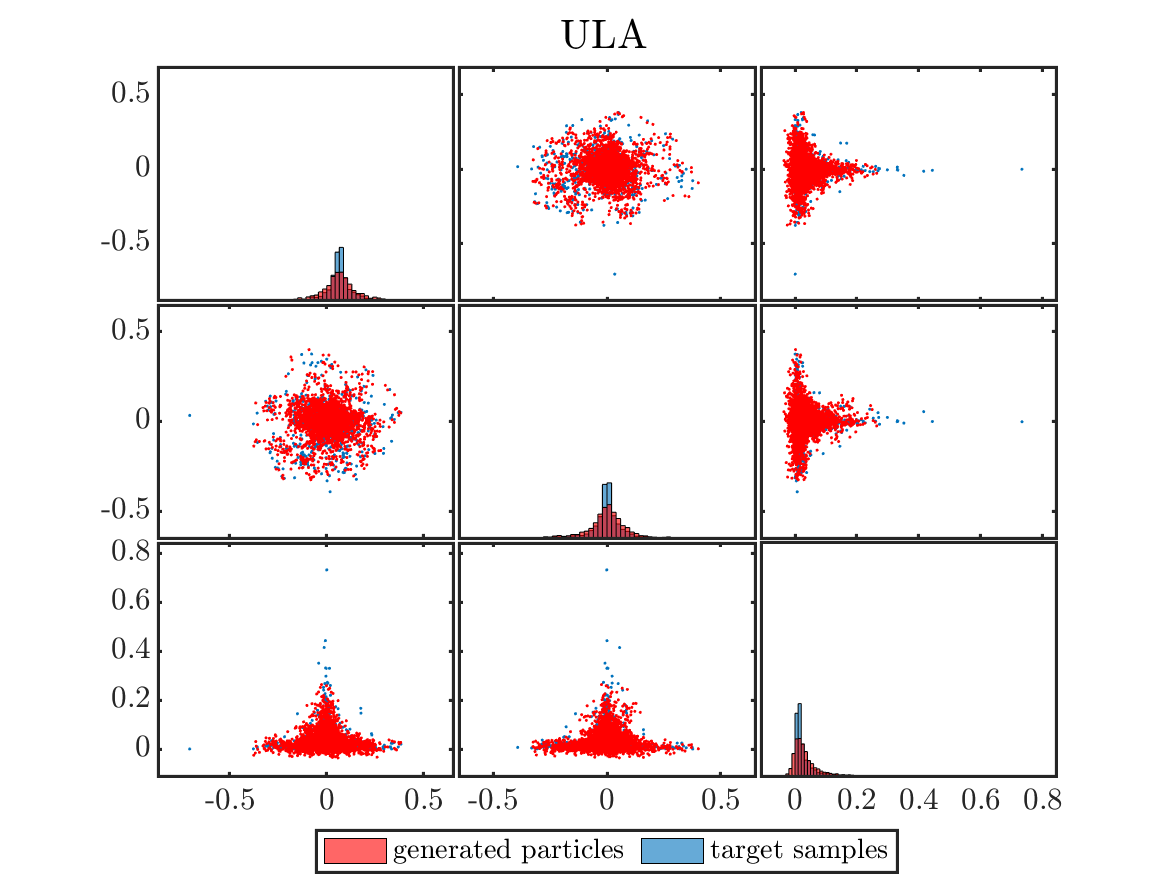}}
\subfigure{\includegraphics[width=0.49\linewidth]{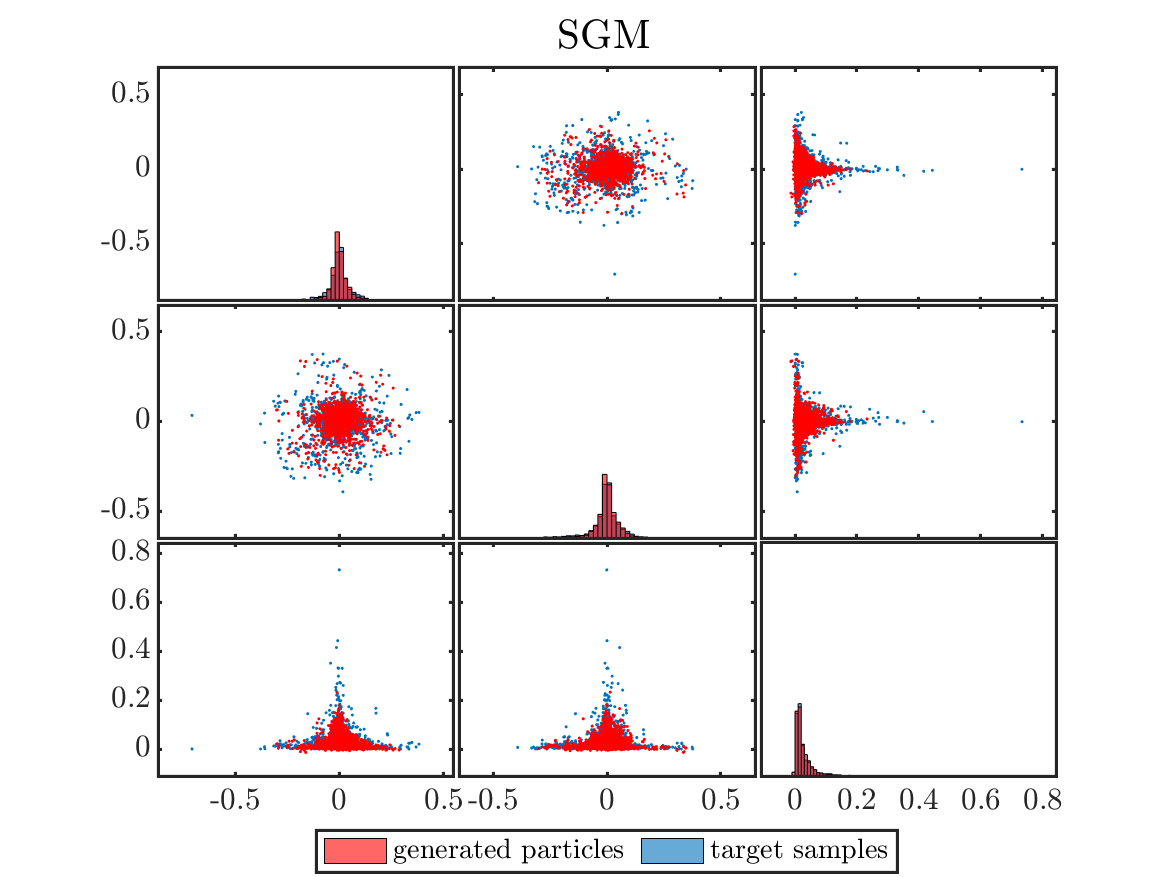}}
\caption{Gluon jet dataset: marginal distributions of the target samples and 2700 generated particles using DMPS, SVGD, ULA, and SGM with 1000 traning samples. Coordinates are $\eta^{\textrm{rel}}$, $\phi^{\textrm{rel}}$, and $p_T^{\textrm{rel}}$, respectively. }
\label{gluon_sample}
\end{figure}

\begin{center}
\begin{table}
\centering
\begin{tabular}{cp{25mm}p{25mm}p{25mm}p{25mm} }  
\hline
\# particles & DMPS & SVGD & ULA &  SGM\\
\hline
100& $\bm{0.0020}$ \newline ${\pm 5.03 \times 10^{-5}}$& $0.0062 \newline \pm 4.00 \times 10^{-4}$& $0.0041 \newline \pm 2.64 \times 10^{-4}$& $0.0029 \newline \pm 1.21 \times 10^{-4}$\\
300& $\bm{0.0014}$ \newline ${\pm 1.65 \times 10^{-5}}$& $0.0059 \newline \pm 3.75 \times 10^{-4}$& $0.0032 \newline \pm 1.73 \times 10^{-4}$& $0.0021 \newline \pm 0.55 \times 10^{-4}$\\
900& $\bm{0.0012}$ \newline ${\pm 1.61 \times 10^{-5}}$& $0.0058 \newline \pm 3.65 \times 10^{-4}$& $0.0029 \newline \pm 1.53 \times 10^{-4}$& $0.0018 \newline \pm 0.46 \times 10^{-4}$\\
2700& $\bm{0.0012}$ \newline ${\pm 1.59 \times 10^{-5}}$& $0.0057 \newline \pm 3.61 \times 10^{-4}$& $0.0027 \newline \pm 1.30 \times 10^{-4}$& $0.0018 \newline \pm 0.44 \times 10^{-4}$\\
\hline
\end{tabular}
\caption{\label{err_table_gluon} Gluon jet dataset: error comparison ($\pm$ standard error) between DMPS, SVGD, ULA, and SGM, for different numbers of generated particles.}
\end{table}
\end{center}

\subsection{Remarks on the experiments}
In these experiments, we see that for all methods the error decreases with more training samples, and that the error of DMPS, ULA, and SGM decreases with more generated particles. We also observe that DMPS has the best performance in terms of the regularized OT metric. The reason might be that the kernel method for approximating the generator relies on diffusion maps, or more broadly, the graph Laplacian, which is widely used for manifold learning due to its flexibility in detecting the underlying geometry. The fact that the generator is approximated as a whole distinguishes it from other score-based methods. 

Although the gradient of the potential term is approximated in the same fashion in our SVGD approach, performance might be affected by the interaction between two kernels: one from the diffusion map approximation and the other being the kernel of the SVGD algorithm itself, as well as the choice of the kernel bandwidths. Performance is more stable for the two non-deterministic methods, ULA and SGM, compared to SVGD. 
For ULA, 
since Gaussian noise is added at in every step, particles are able to ``leave'' the manifold for any finite stepsize. On the other hand, SGM produces generally good results. However, it does suffer from having longer training times. For the arc problem, with $1000$ training samples and $100$ initial particles, DMPS took $6.21$ seconds while SGM (trained using 5000 epochs) took $1633$ seconds ($27$ minutes).

\section{Conclusion} \label{sec:conclusion}
We introduced DMPS as a \emph{simple-to-implement} and \emph{computationally efficient} kernel method for generative modeling. Our approach combines diffusion maps with the LAWGD approach to construct a generative particle system that adapts to the geometry of the underlying distribution. Our method compares favorably with other competing schemes (SVGD and ULA with learned scores, and diffusion-based generative models) on synthetic datasets, consistently achieving the smallest errors in terms of regularized OT distance. While the examples presented here are of moderate dimension (up to $d=15$ for the example in Section~\ref{sec:semisphere}), we expect that more sophisticated kernel methods \cite{Li2017kernel} 
can slot naturally into the DMPS framework and help extend the method to higher-dimensional problems.
Future research will also study the convergence rate of the method in discrete time and with finite samples, and consider more complex geometric domains. 


\section*{Acknowledgements}

This work was supported in part by the US Department of Energy, Office of Advanced Scientific Computing Research, under award numbers DE-SC0023187 and DE-SC0023188.

\bibliographystyle{plainnat}
\bibliography{reference}

\newpage
\appendix
\onecolumn



\newcommand{\dom}{\mathfrak{D}}

\section{Spectral properties of \texorpdfstring{$\mathscr L$}{the generator}}\label{appendix:spectrum}
\rev{Suppose $\mathscr L = -\nabla^2 +\langle \nabla V, \nabla \cdot \rangle$ is an operator defined on a domain $\dom \subset L^2(\pi)$ with discrete spectrum $\lambda_0 = 0 < \lambda_1 \leq \lambda_2 \leq \cdots $. Let $\phi_i$ be the corresponding eigenfunctions. The action of $\mathscr L: \dom \rightarrow L^2(\pi)$ reads 
\begin{align*}
    \mathscr L f(x) = \sum_{i = 1}^\infty \lambda_i \langle \phi_i, f \rangle_{L^2(\pi)} \phi_i(x),
\end{align*}
where, since $\lambda_0 = 0$, the summation above could equivalently start at $i=0$ or at $i=1$. 
We define $\mathscr{L}^{-1}$ via
\begin{align*}
    \mathscr L^{-1} f(x) \coloneqq \sum_{i = 1}^\infty \lambda_i^{-1} \langle \phi_i, f \rangle_{L^2(\pi)} \phi_i(x).
\end{align*}
Note that we exclude the eigenfunction corresponding to $\lambda_0 = 0$. Therefore, if we write 
\begin{align*}
    f = \langle \phi_0, f \rangle_{L^2(\pi)} \phi_0 + \sum_{i = 1}^\infty \langle \phi_i, f \rangle_{L^2(\pi)} \phi_i=\int f d\pi + \sum_{i = 1}^\infty \langle \phi_i, f \rangle_{L^2(\pi)} \phi_i,
\end{align*}
then 
\begin{align*}
    \mathscr L^{-1} \left(f - \int f d \pi \right) = \sum_{i = 1}^\infty \lambda_i^{-1} \langle \phi_i, f \rangle_{L^2(\pi)} \phi_i = \mathscr L^{-1} f,
\end{align*}
where we have used the fact that $\phi_0 = 1$ is the eigenfunction corresponding to $\lambda_0 = 0$.
Note that $\mathscr L^{-1} \mathscr L f = f - \int fd\pi$. Therefore, if $\int fd\pi = 0$, $\mathscr L^{-1} \mathscr L f = f$, and  $\mathscr{L}^{-1}$ can be understood as the inverse of $\mathscr{L}$ on the space orthogonal to the constant functions.}

\section{Finite sample approximation of the operator}\label{appendix:finite_sample_kernel}
In this section, we introduce the finite sample counterpart to the approximations in Section~\ref{subsec:kernel_L}. We consider the approximation of the generator of the Langevin diffusion process $\mathscr L$ from finite samples $\{z^i\}_{i = 1}^N \sim \pi$. In this case, we have 
\begin{align*}
    M_{\epsilon,N}(x, y) = \frac{K_\epsilon(x,y)}{\sqrt{\sum_{i = 1}^N K_\epsilon(z^i,y)} \sqrt{\sum_{i=1}^N K_\epsilon(x,z^i)}},
\end{align*}
and the corresponding $L_{\epsilon,N}^f$ and $L_{\epsilon,N}^b$
can be written as 
\begin{align*}
    P_{\epsilon, N}^f(x,y) := \frac{M_{\epsilon,N}(x,y)}{\sum_{i=1}^N M_{\epsilon,N}(z^i,y) },\\
    P_{\epsilon, N}^b(x,y) := \frac{M_{\epsilon,N}(x,y)}{\sum_{i=1}^N M_{\epsilon,N}(x,z^i)},
\end{align*}
and we set 
\begin{align}
    P_{\epsilon, N}(x,y) = \frac 12 \left(P_{\epsilon, N}^f(x,y) + P_{\epsilon, N}^b(x,y)\right). \label{finite_P}
\end{align}
Similarly, its action on a function $g$ writes
\begin{align*}
    T^{f,b}_{\epsilon,N} g(x) = \sum_{i = 1}^N P^{f,b}_{\epsilon,N}(x,z^i) g(z^i).
\end{align*}
Let
\begin{align*}
    L^{f,b}_{\epsilon,N} :=\frac{\mathrm{Id} - T^{f,b}_{\epsilon,N} }{\epsilon}.
\end{align*}
Similar to their spatial continuum limit, we have
\begin{align*}
    \lim_{\epsilon \rightarrow 0, N\rightarrow \infty} L^{f}_{\epsilon,N} = \lim_{\epsilon \rightarrow 0, N\rightarrow \infty} L^{b}_{\epsilon,N} = \mathscr L. 
\end{align*}
Then, similar to \eqref{approx_kernel}, we have that 
\begin{align*}
     \lim_{\epsilon \rightarrow 0, N \rightarrow \infty}\frac{g(x) - \sum_{i = 1}^N  P^{f,b}_{\epsilon,N}(x,z^i) g(z^i)  }{\epsilon} = \mathscr L g(x), 
\end{align*}
and
\begin{align*}
    L_{\epsilon,N} = \frac 12 (L_{\epsilon,N}^f + L_{\epsilon,N}^b).
\end{align*}
is the symmetric kernel.

\subsection{Computing $\nabla_1 P_{\epsilon}(x,y)$}\label{analytic_grad}
Recall that $P_\epsilon = \frac 12 (P_\epsilon^f + P_\epsilon^b)$. We then compute $\nabla_1 P_\epsilon^f(x, y)$ and $\nabla_1 P_\epsilon^b(x, y)$ separately. 
\subsubsection{Computing $\nabla_1 P^f_{\epsilon}(x,y)$}
Recall 
\begin{align*}
    P_{\epsilon}^f(x,y) := \frac{M_\epsilon(x,y)}{\int_X M_\epsilon(x,y) d\pi(x)},
\end{align*}
where 
\begin{align}\label{grad_M}
    M_\epsilon(x,y):=\frac{K_\epsilon(x,y)}{\sqrt{\int_{\mathcal{X}} K_\epsilon(x,y) d\pi(x)} \sqrt{\int_{\mathcal{Y}} K_\epsilon(x,y) d\pi(y)}}.
\end{align}
Then 
\begin{align*}
    \nabla_1 P_{\epsilon}^f(x,y) = \frac{\nabla_1 M_\epsilon(x,y) }{\int_X M_\epsilon(x,y) d\pi(x)}.
\end{align*}
We then compute $\nabla_1 M_\epsilon(x,y)$. Let $d_\epsilon(x) = \int_{\mathcal{Y}} K_\epsilon(x,y) d\pi(y)$ and $d_\epsilon(y) = \int_{\mathcal{X}} K_\epsilon(x,y) d\pi(x)$, then
\begin{align*}
    \nabla_1 M_\epsilon(x,y) = \frac{\nabla_1 K_\epsilon(x,y)\sqrt{d_\epsilon(y)} \sqrt{d_\epsilon(x)} - \left(\frac{\partial }{\partial x}\sqrt{d_\epsilon(x)}\right)\sqrt{d_\epsilon(y)}K_\epsilon(x,y)}  {d_\epsilon(x) d_\epsilon(y)},
\end{align*}
where 
\begin{align}
    \nabla_1 K_\epsilon(x,y) &= -\left(\frac{x - y}{\epsilon} \right) e^{-\frac{\left\Vert x - y \right\Vert^2}{2\epsilon}}, \label{grad_K}\\
     \frac{\partial}{\partial x} \left(\sqrt{d_\epsilon(x)} \right) & = \frac{1}{2} \int_{\mathcal Y} \left(-\left(\frac{x - y}{\epsilon} \right) e^{-\frac{\left\Vert x - y \right\Vert^2}{2\epsilon}}\right)^{-1/2} d\pi(y). \label{grad_sqrt_d}
\end{align}

\subsubsection{Computing $\nabla_1 P_{\epsilon}^b(x,y)$}
On the other hand, we have 
\begin{align*}
    P_{\epsilon}^b(x,y) := \frac{M_\epsilon(x,y)}{\int_\mathcal Y M_\epsilon(x,y) d\pi(y)},
\end{align*}
and
\begin{align*}
    \nabla_1 P_{\epsilon}^b(x,y) = \frac{\nabla_1 M_\epsilon(x,y) \int_\mathcal Y M_\epsilon(x,y) d\pi(y) - \int_\mathcal Y \nabla_1 M_\epsilon(x,y) d\pi(y) M_\epsilon(x,y)}{\left( \int_\mathcal Y M_\epsilon(x,y) d\pi(y)\right)^2},
\end{align*}
where all the ingredients are computable from \eqref{grad_M}, \eqref{grad_K}, and \eqref{grad_sqrt_d}.

\subsection{Computing $\nabla_1 P_{\epsilon, N}(x,y)$}
The discrete version is obtained by replacing the integral with its empirical average. Similarly, let $\{z^i\}_{i = 1}^N\sim\pi$, and define $d_{\epsilon,N}(x) = \sum_{i=1}^N K_\epsilon(x,z^i)$ and $d_{\epsilon,N}(y) = \sum_{i = 1}^N K_\epsilon(z^i,y)$. Recall that 
\begin{align*}
    M_{\epsilon,N}(x, y) = \frac{K_\epsilon(x,y)}{\sqrt{\sum_{i = 1}^N K_\epsilon(z^i,y)} \sqrt{\sum_{i=1}^N K_\epsilon(x,z^i)}}.
\end{align*}
Then
\begin{align*}
    \nabla_1 M_{\epsilon,N}(x,y) = \frac{\nabla_1 K_{\epsilon,N}(x,y)\sqrt{d_{\epsilon,N}(y)} \sqrt{d_{\epsilon,N}(x)} - \left(\frac{\partial }{\partial x}\sqrt{d_{\epsilon,N}(x)}\right)\sqrt{d_{\epsilon,N}(y)}K_\epsilon(x,y)}  {d_{\epsilon,N}(x) d_{\epsilon,N}(y)},
\end{align*}
and 
\begin{align*}
    \frac{\partial}{\partial x} \left(\sqrt{d_{\epsilon,N}(x)} \right) & = \frac{1}{2} \sum_{i = 1}^N \left(-\left(\frac{x - z^i}{\epsilon} \right) e^{-\frac{\left\Vert x - z^i \right\Vert^2}{2\epsilon}}\right)^{-1/2}.\\
\end{align*}
Finally, similar to the previous section, we have that 
\begin{align*}
    \nabla_1 P_{\epsilon,N}^f(x,y) = \frac{\nabla_1 M_{\epsilon,N}(x,y) }{\sum_{i=1}^N M_{\epsilon,N}(z^i,y) },
\end{align*}
and
\begin{align*}
    \nabla_1 P_{\epsilon,N}^b(x,y) = \frac{\nabla_1 M_{\epsilon,N}(x,y) \left(\sum_{i = 1}^N M_{\epsilon,N}(x,z^i)\right)  - \left(\sum_{i = 1}^N \nabla_1 M_{\epsilon,N}(x,z^i)\right)  M_{\epsilon,N}(x,y)}{\left(\sum_{i = 1}^N M_{\epsilon,N}(x,z^i) \right)^2}.
\end{align*}

\section{Symmetric positive-definiteness of $P_\epsilon$}\label{pd_kernel}
In this section we show that $P_\epsilon = \frac 12 (P_\epsilon^f + P_\epsilon^b)$ is symmetric positive definite. 
\begin{theorem}
    $P_\epsilon: \mathcal{D} \times \mathcal{D} \to \mathbb{R}$ is a positive definite kernel. 
\end{theorem}
\begin{proof}
Recall that 
\begin{align*}
    P_{\epsilon}^f(x,y) := \frac{M_\epsilon(x,y)}{\int M_\epsilon(x,y) d\pi(x)},\\
    P_{\epsilon}^b(x,y) := \frac{M_\epsilon(x,y)}{\int M_\epsilon(x,y) d\pi(y)},
\end{align*}
where
\begin{align*}
    M_\epsilon(x,y):=\frac{K_\epsilon(x,y)}{\sqrt{\int K_\epsilon(x,y) d\pi(x)} \sqrt{\int K_\epsilon(x,y) d\pi(y)}},
\end{align*}
and $K_\epsilon$ is the Gaussian kernel, which is symmetric positive definite. Therefore, we have
\begin{align*}
    \sum_{i=1}^n \sum_{j=1}^n c_i c_j K_\epsilon(x^i, x^j) > 0,
\end{align*}
for all $n\in \mathbb N$, $x^1, \ldots, x^n \in \mathcal D$, and $c_1, \ldots, c_n \in \mathbb R$. This is equivalent to saying that $K$, with $K_{ij} = K_\epsilon(x^i, x^j)$, has positive eigenvalues for all choices of $n\in \mathbb N$, $x^1, \ldots, x^n \in \mathcal D$. 

Then we see immediately that $M_\epsilon(x^i, x^j)$ has positive eigenvalues, as we can write $M = DKD$, where $M, K, D \in \mathbb R^{n\times n}$,  $M_{ij} = M_\epsilon(x^i, x^j)$, and $D$ is a diagonal matrix with positive diagonal entries $D_{ii} = 1/\sqrt{\int K_\epsilon(x^i,z) d\pi(z)}$. Using the same argument, we have that $P^f_\epsilon$ is also positive definite as $P^f = MQ^f $, where $P^f, Q^f \in \mathbb R^{n\times n}$ and $Q^f$ is a diagonal matrix with $Q^f_{ii} = 1/\sqrt{\int M_\epsilon(z,x^i) d\pi(z)}$. $Q^f$ has positive diagonal entries, as $M_\epsilon(x,y)$ is positive for all $x,y \in \mathcal D$. Similarly, $P^b$ also has positive eigenvalues. Therefore, we conclude that $P_\epsilon$ is symmetric positive definite. 

\end{proof}

\color{black}

\section{Regularity assumptions} \label{regularity_assump}
We state here the regularity assumptions needed for the gradient to converge. The statement and the proof are adapted from \citet{stack_ex}.

\begin{theorem}
    Suppose $L_\epsilon f(x)$ is a family of bounded differentiable functions from $\mathcal D$ to $\mathbb R$ converging pointwise to $\mathscr L f(x)$ as $\epsilon \rightarrow 0$. Furthermore, suppose $\nabla L_\epsilon f(x)$ is a family of uniformly equicontinuous functions. Then $\mathscr L f(x)$ is differentiable on $\mathcal D$ and $\nabla L_\epsilon f(x)$ converges to $\nabla \mathscr L f(x)$ uniformly.
\end{theorem}
\begin{proof}
We first choose a countable set of $\epsilon$, say, $\epsilon = \{1/n\}_{i = 1}^\infty$, and we use $L_n$ to denote $L_{\epsilon = 1/n}$ for convenience. Since $L_n f(x)$ are uniformly bounded and $\nabla L_n f(x)$ are uniformly equicontinuous,  $\nabla L_n f(x)$ are uniformly bounded. Then $\nabla L_n f(x)$ has a subsequence $\nabla L_{n(k)} f(x)$ that converges uniformly to some function $g\in C(\mathcal D)$ by the Arzela-Ascoli theorem. We then show that $g = \nabla \mathscr L f(x)$ by contradiction. Suppose $\nabla L_n f(x)$ does not converge uniformly to $\nabla \mathscr L f(x)$. Then there exists $\epsilon >0$ and another subsequence $\nabla L_{m(k)} f(x)$ of $\nabla L_{n} f(x)$ such that $\Vert \nabla L_{m(k)} f - \nabla \mathscr L f\Vert_\infty > \epsilon$ for all $k$. But by the Arzela-Ascoli theorem, $\nabla L_{m(k)}$ has a subsequence converging uniformly to $\nabla \mathscr L f(x)$: contradiction. Therefore, $\nabla L_n f(x)$ converges to $\nabla \mathscr L f(x)$ uniformly on $\mathcal D$.
\end{proof}

\medskip
Received xxxx 20xx; revised xxxx 20xx; early access xxxx 20xx.
\medskip

\end{document}